\documentclass{article}

\usepackage{microtype}
\usepackage{graphicx}
\usepackage{placeins}
\usepackage{booktabs} 

\usepackage{hyperref}
\hypersetup{
  bookmarksnumbered,
  colorlinks=false,
  allcolors=blue,
  pdfauthor={Aurélien Pion and Emmanuel Vazquez},
  pdftitle={Goal-Oriented Lower-Tail Calibration of Gaussian Processes for Bayesian Optimization},
  pdfsubject={},
  pdfkeywords={},
  pdfpagemode=UseOutlines
}
\usepackage{subcaption}
\usepackage[normalem]{ulem} 
\usepackage{multirow}
\usepackage{enumitem}
\usepackage{tabularx}
\usepackage[table]{xcolor}
\usepackage{colortbl}
\usepackage{array}

\usepackage[accepted]{icml2026}

\usepackage{amsmath}
\usepackage{amssymb}
\usepackage{mathtools}
\usepackage{amsthm}
\usepackage{dsfont}
\usepackage{xspace}
\usepackage{xfrac}
\usepackage[capitalize,noabbrev]{cleveref}

\newcolumntype{Y}{>{\raggedright\arraybackslash}X}
\newcolumntype{s}{>{\normalsize\centering\arraybackslash}X}

\ifpdf
  \DeclareGraphicsExtensions{.eps,.pdf,.png,.jpg}
\else
  \DeclareGraphicsExtensions{.eps}
\fi

\usepackage{amsopn}

\DeclareMathOperator \GP {{\mathrm GP}} \DeclareMathOperator \Var {{\mathrm var}}
\DeclareMathOperator*{\argmin}{arg\,min}

\newcommand \Dcal   {\mathcal{D}}
 
\newcommand \E      {\mathsf{E}}

\newcommand \X      {\mathbb{X}}

\newcommand \R      {\mathbb{R}}

\newcommand \N      {\mathbb{N}}
\renewcommand \P    {\mathsf{P}}

\newcommand \XX     {\mathbb{X}}
\newcommand {\one}  {\mathds{1}}

\providecommand \Var[1]  {\operatorname{\mathsf{Var}}}

\DeclareMathOperator*{\argmax}{arg\,max}

\newcommand{\fnsb}[1]{\uline{{#1}}}
\theoremstyle{plain}
\newtheorem{theorem}{Theorem}[section]
\newtheorem{proposition}[theorem]{Proposition}
\newtheorem{lemma}[theorem]{Lemma}

\theoremstyle{definition}

\theoremstyle{remark}
\newtheorem{remark}[theorem]{Remark}

\newcommand{\tcgp}{\textsc{tcGP}\xspace}
\newcommand{\tcgploc}{\textsc{tcGP}-thres\xspace}
\newcommand{\tcgpmar}{\textsc{tcGP}-occ\xspace}

\newcommand{\bcrgp}{\textsc{bcrGP}\xspace}
\newcommand{\regp}{\textsc{reGP}\xspace}
\newcommand{\onego}{on\textsc{GP}\xspace}
\usepackage[textsize=tiny]{todonotes}

\icmltitlerunning{Goal-Oriented Calibration of GP for BO}

\usepackage{xcolor}
\usepackage[normalem]{ulem}
\definecolor{redviolet}{RGB}{199,21,133}

\begin{document}

\twocolumn[
\icmltitle{Goal-Oriented Lower-Tail Calibration of Gaussian Processes \\
  for Bayesian Optimization}

\icmlsetsymbol{equal}{*}

\begin{icmlauthorlist}
\icmlauthor{Aurélien Pion}{tsv,l2s}
\icmlauthor{Emmanuel Vazquez}{l2s}
\end{icmlauthorlist}

\icmlaffiliation{tsv}{Transvalor S.A., Biot, France}
\icmlaffiliation{l2s}{Univ. Paris-Saclay, CNRS,
    CentraleSupélec, L2S, Gif-sur-Yvette, France}

\icmlcorrespondingauthor{Emmanuel Vazquez}{emmanuel.vazquez@centralesupelec.fr}

\icmlkeywords{Machine Learning, ICML}

\vskip 0.3in
]

\printAffiliationsAndNotice{}  

\begin{abstract}
Gaussian process (GP) predictive distributions are commonly used in Bayesian
optimization (BO) to guide the selection of evaluation points for expensive
objective functions. 
The choice of kernel and hyperparameters has a strong influence on the
exploration--exploitation trade-off. 
For minimization, sampling criteria such as expected improvement (EI) depend on
both the probability mass below the current best value and the shape of the
predictive distribution in this region. 
This article studies goal-oriented calibration of GP predictive
distributions below a low threshold~$t$ in the noiseless setting, for
standard GP models with hyperparameters selected by maximum
likelihood. 
We consider two complementary forms of calibration below~$t$ for inputs
distributed according to a reference measure~$\mu$: occurrence calibration over
the design space and
thresholded $\mu$-calibration on sublevel sets of the form
$\{x\in\mathbb{X}, f(x)\le t\}$. We propose tcGP, a post-hoc method that combines
these two forms of calibration for GP predictive distributions below~$t$. With
fixed GP hyperparameters, the exact EI
sampling criterion based on tcGP generates a sequence of evaluation
points that is dense in the design space. Experiments on standard
benchmarks show improved lower-tail calibration and BO performance
relative to standard GP models and globally calibrated GP models.
\end{abstract}

\section{Introduction}

This article considers the minimization of an
\emph{expensive-to-evaluate} function $f$ defined over a
design space $\X \subset~\mathbb{R}^d$, with noiseless evaluations
(a query at $x$ returns $f(x)$ exactly). 
A widely used approach for this setting is Bayesian optimization (BO)
\citep[see, e.g.,][]{mockus:1978, Jones1998:article_ego,
  villemonteix_informational_2009, srinivas2010:_ucb}. 
In the single-objective setting, BO aims at identifying a global
minimizer $x^\star \in \arg\min_{x \in \X} f(x)$ using a limited
evaluation budget, modeling $f$ a priori as a Gaussian process (GP),
denoted by $\xi \sim \GP(m,k)$, with mean function $m$ and covariance
kernel $k$. 
Given an initial set of evaluations, the GP posterior yields
predictive distributions for $f$: for each $x\in\X$, it provides a
predictive cumulative distribution function (CDF)
$\hat F_n(\cdot\mid x)$ for the unknown value $f(x)$ at iteration
$n$. 
At each iteration, the next evaluation point is selected by maximizing
a \emph{sampling criterion}, also referred to as an \emph{acquisition
  function}, that depends on this predictive distribution, such as
expected improvement (EI) \citep{mockus:1978, schonlau:96:gonff}.
 
To avoid unnecessary evaluations of $f$, the predictive distributions
should reflect uncertainty at each iteration $n$ in order to guide
the selection of new evaluation points toward the minimum.
In practice, GP predictive distributions are sensitive to kernel choice and
hyperparameter misspecification, which can lead to over- or under-estimated
uncertainty \citep{pion:2025}.
Such misspecification can mislead the sampling criterion and substantially
degrade the performance of BO \citep{Bogunovic:2021}.

While improving GP predictive performance over the whole domain is desirable
\cite{Bogunovic:2021, Guo2021CalibratedEI, tuo2022uncertainty, deshpande:2024, 
tom:2025}, 
BO can benefit from \emph{goal-oriented} prediction, in particular from accurate
predictions below a low threshold $t\in\mathbb{R}$. 
Here, $t$ is a data-driven threshold chosen in the lower tail of the
observations, e.g., as an empirical~$\delta$-quantile of the observed responses
for some $\delta\in(0,1]$. 
Along these lines, \citet{petit:2025} propose the \emph{relaxed Gaussian process}
(\regp), a GP-based model designed to improve predictive accuracy below $t$. 
\regp builds on GP interpolation and relaxes the interpolation constraints above~$t$
to improve predictions below this level. 
They report improved BO performance compared with standard GP-based methods.

\begin{figure*}[h!]
    \centering
    \includegraphics[width=\textwidth]{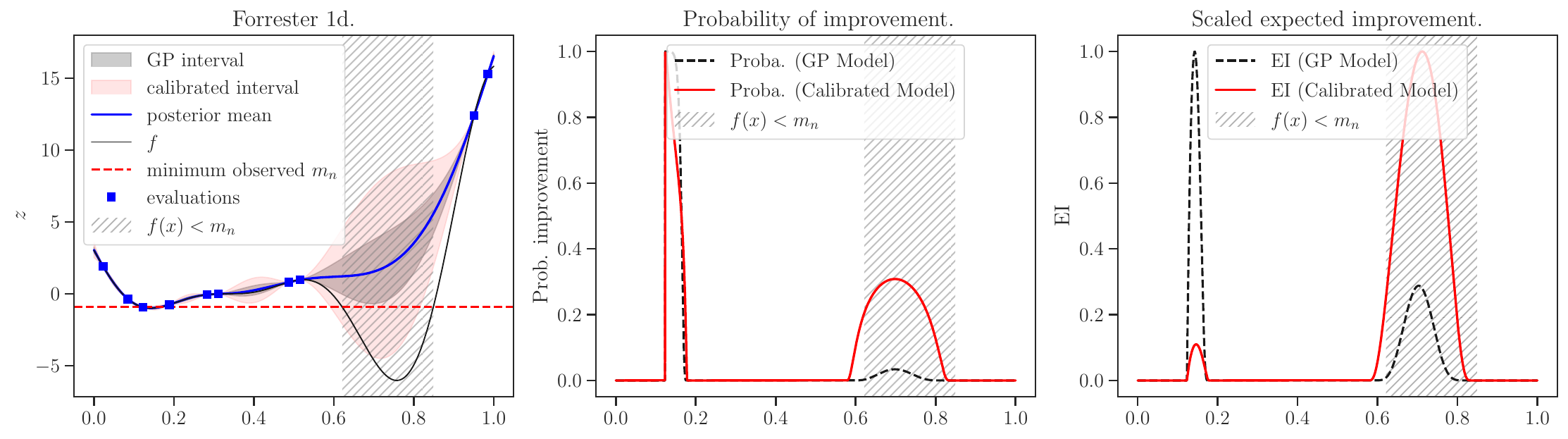}
    \caption{Comparison between a standard GP and a model calibrated
      below $q_{\delta,n}$ (\tcgp), with $\delta=0.3$, on a standard
      one-dimensional test function and with $n=10$ evaluations.  Left:
      observations with GP predictions; the current design does not
      include a point near the global minimizer.  Middle:
      probabilities of improvement $\hat F_{n}^{(0)}(m_n\mid x)$ (GP,
      black) and $\hat F_{n}^{(1)}(m_n\mid x)$ (\tcgp, red), where
      $m_n=\min_{i\le n} f(X_i)$.  Right: EI for the GP and \tcgp,
      each rescaled by its maximum over $\X$. Here, calibration raises the
      predicted probability of improvement in the sparsely sampled region near
      the global minimizer, where the current design has no evaluation, and EI
      increases there accordingly. This exploration behavior seems appropriate
      in this example.}
    \label{fig:intro_example}
\end{figure*}

Following the idea of goal-oriented prediction below a threshold~$t$,
this work investigates whether BO can benefit from probabilistic models that are
\emph{calibrated below a low threshold} $t$, and compares them with standard GP
models and with post-hoc globally calibrated GP predictive distributions over
$\X$. 
By calibration below~$t$, we mean that the lower tail of the predictive CDFs
$\hat F_n(\cdot\mid x)$ is reliable for inputs drawn from $\mu$: the average
probability assigned to the event~$f(X)\le t$, with $X\sim\mu$, matches its
frequency, and, on the sublevel set $\{f(X)\le t\}$, the transformed lower-tail
values have the correct distribution.

Consider Figure~\ref{fig:intro_example}, which compares a GP model with
hyperparameters selected by maximum likelihood and a model calibrated
below the empirical $\delta$-quantile $q_{\delta,n}$ ($\delta=0.3$), on a standard
one-dimensional test function \citep[from][]{forrester:2008}. 
The middle panel displays the predicted probabilities of improvement
\begin{equation}
\hat F_{n}^{(j)}(m_n\mid x),\qquad j\in\{0,1\},
\end{equation}
where $\hat F_n^{(j)}(\cdot\mid x)$ is the predictive CDF at $x$, with
$j=0$ for the GP and $j=1$ for the calibrated predictive distribution,
and $m_n=\min_{i\le n} f(X_i)$. 
In the shaded region, the GP assigns low lower-tail probabilities despite sparse
sampling, while the calibrated model assigns larger values, which raises EI in
that region. 
This example is constructed to illustrate the issue in a minimal setting. 
In higher dimensions, similar behavior can arise when the selected
hyperparameters lead to local overconfidence in sparsely sampled regions, with
underestimated lower-tail probabilities and reduced exploration. 
This motivates lower-tail calibration of predictive distributions, since
improvement-based sampling criteria are sensitive to tail miscalibration.

This work makes two contributions. 
Inspired by the notion of \emph{probabilistic tail calibration} introduced by
\citet{allen:2025}, we define two notions of \emph{spatial calibration below~$t$}:
\emph{occurrence calibration} over $\X$ and \emph{thresholded~$\mu$-calibration}
on the sublevel set $\{x\in\X, f(x)\le t\}$. 

A \emph{post-hoc} calibration method for GP predictive distributions
is also introduced, inspired by \citet{pion:2025}, termed
\emph{calibrated prediction below~$t$ for GPs} (\tcgp). 
Standardized residuals are modeled with a parametric generalized
normal family, with parameters selected by optimizing a criterion
targeting spatial calibration below~$t$. 
The resulting predictive distributions are used in BO with the EI
criterion. 
Analytical expressions for EI with \tcgp are derived, together with a
convergence result. 
We also report additional experiments using the UCB sampling criterion
in the appendix.

We evaluate \tcgp and \regp empirically and find improved lower-tail calibration and
BO performance relative to standard GP models, globally calibrated GP models
\citep{pion:2025}, sequential calibration \citep{deshpande:2024}, and variants that
target only one of the two calibration notions.

Section~\ref{sec:related_work} reviews related
work. Sections~\ref{sec:bg-setting}–\ref{sec:goal-orient-calibr}
introduce the setting, \tcgp, and the \tcgp-based EI
algorithm. Section~\ref{sec:experiment_calibration} reports
experiments.

\section{Related Work}
\label{sec:related_work}

\paragraph{Calibration and goal-oriented prediction}
Classical notions of calibration, such as probabilistic calibration
and proper scoring rules, typically address \emph{global} calibration
of predictive distributions over the domain~$\X$
\citep{GneitingRaftery2007, Gneiting:2023, pion:2025}. 
Calibration in BO with non-i.i.d.\ data is studied by
\citet{deshpande:2024}, who propose online recalibration and conformal procedures
to maintain quantile calibration within the BO loop. 
\citet{Sahoo:2021} introduced \emph{threshold calibration} for
decision-making applications and proposed a method for obtaining reliable
estimates of threshold-based losses. 
Their notion is closely related to
thresholded $\mu$-calibration considered here. This article is formulated in
a spatial setting over $\X$ and considers two notions of calibration below~$t$:
occurrence calibration over $\X$ and thresholded $\mu$-calibration on the sublevel set
$\{x\in\X, f(x)\le t\}$. 
\citet{friedli:2025} introduced a GP-based sequential design method built from
the threshold-weighted continuous ranked probability score (twCRPS), and applied
it to excursion set estimation.

\paragraph{Conformal prediction.}
Conformal prediction (CP) provides finite-sample coverage guarantees
for a new response when the data are exchangeable
\citep{vovk_Gammerman_2005}. Conformal predictive systems (CPS) extend
CP to full predictive distributions \citep{vovk19:_nonpar}. CPS have
been applied to GP interpolation by \citet{pion:2025}. 
Goal-oriented CP methods target coverage on selected subsets of inputs
\citep{zhang_posterior_2024, jin2024confidencefocalconformalprediction}, but 
typically yield piecewise-constant predictive distributions and still rely on
exchangeability. CP has also been used for sequential optimization
with noisy evaluations \citep{stanton:2023, kim:2025}. These
approaches are not tailored to lower-tail BO criteria and can be
computationally demanding. In this work, we adopt a GP-specific
post-hoc calibration approach inspired by \citet{pion:2025} and adapt
it to prediction below a threshold.

\paragraph{Alternative BO formulations.}
Misspecification in BO and sequential decision-making has been studied by
\citet{Neiswanger:2021}, who analyzed the impact of global variance
misspecification. Several works also modify the probabilistic model or the
sampling criterion to improve BO, without targeting calibration of predictive
lower-tail probabilities. \citet{picheny_2019_ordinalbayesianoptimisation}
considers BO with \emph{ordinal observations} (e.g., pairwise preferences or
rank information) and builds on variational ordinal GP regression
\citep{JMLR:v6:chu05a}. \citet{picheny_2022_bayesianquantileexpectileoptimisation}
focuses on quantiles and expectiles of the predictive distribution.

\section{Background and Problem Setup}
\label{sec:bg-setting}

\subsection{Problem Setting and Notation}

Bayesian optimization proceeds sequentially.  We start from an initial
dataset $\Dcal_{n_0} = \{(X_i, Z_i)\}_{i=1}^{n_0}$, where
$Z_i = f(X_i)$ and the inputs $X_i$ are drawn according to an initial
design measure $\mu_0$ on $\X$. We then run $n - n_{0}$ additional BO
iterations, $n > n_0$. For $i > n_0$, the design points $X_i$ are selected by
maximizing a sampling criterion, so their distribution is determined by the BO
strategy. Using the
dataset~$\Dcal_n$ we construct a family of predictive CDFs
$\hat F_n(\cdot \mid x)$ for~$x\in \X$. This
family is used to define the sampling criterion and thereby governs the choice 
of the next point $X_{n+1}$ to evaluate $f$.

We focus on the EI sampling criterion in this work. 
(We also consider the UCB policy and report complementary results in
Appendix~\ref{app:add_res_ucb}.) 
For the family of predictive CDFs $\hat F_n(z\mid x)$, with
$x\in \X$, the EI is defined as
$$
  \rho_n(x)
  = \int_{\R} (m_n - z)_+ \,\mathrm d\hat F_n(z\mid x),
$$
where $m_n = \min(f(X_1), \ldots, f(X_n))$. EI is primarily controlled
by the predicted probability mass below $m_n$ and the shape of the
predictive CDF on $(-\infty, m_n]$, motivating calibration
specifically in the lower tail rather than globally.

Consider a fixed reference measure $\mu$ on $\X$ used to define and assess
spatial calibration (e.g., uniform on $\X$). 
The measure $\mu$ may coincide with $\mu_0$, but it is fixed
independently of the BO policy. 
For fixed $n$ and $\Dcal_n$, let~$X\sim\mu$ be an auxiliary draw,
independent of $\Dcal_n$. 
We write~$\P_n(\cdot):=\P_X(\cdot)$ and $\E_n(\cdot):=\E_X(\cdot)$ for
probability and expectation with respect to this auxiliary draw $X\sim\mu$,
with~$\Dcal_n$ fixed. The subscript indicates that the integrands depend on
$\Dcal_n$ through $\hat F_n(\cdot\mid X)$.

\subsection{A Generalized Normal Distribution Model for Prediction Errors}
\label{sec:bcr_gp}

In this section, we depart from the Bayesian interpretation of the GP model.
The function $f$ is treated as fixed and deterministic, and the GP formalism is
used only as a device to construct, from $\Dcal_n$, a predictive mean $f_n(\cdot)$ and
a predictive variance $\sigma_n^2(\cdot)$. We retain from this construction the
associated Gaussian predictive
distributions~$\hat F_n^{\mathrm{GP}}(z\mid x) = \Phi \bigl((z-f_n(x)) / \sigma_n(x)\bigr)$, $x\in\X$, $z\in\R$,
where $\Phi$ is the standard normal CDF,
with the convention~$\hat F_n^{\mathrm{GP}}(z\mid x)=\mathbf{1}\{z\ge f_n(x)\}$
when $\sigma_n(x)=0$.

To assess deviations from this Gaussian predictive distribution for $X\sim\mu$, define
the standardized prediction error
\begin{equation}
R_n(x,f(x))=\frac{f(x)-f_n(x)}{\sigma_n(x)},
\end{equation}
with $R_n(x,f(x))=0$ if $\sigma_n(x)=0$.
Under the GP model, for any fixed $x$ with $\sigma_n(x)>0$, the conditional
distribution of $(\xi(x)-f_n(x))/\sigma_n(x)$ given $\Dcal_n$ is standard normal.
In contrast, since $f$ is deterministic, $R_n(X,f(X))$ is random only through
$X\sim\mu$ and its distribution need not be Gaussian.

Following \citet{pion:2025}, we model $R_n(X,f(X))$ with a generalized normal (GN)
family, denoted by~$\mathcal{GN}(\beta,l,\lambda)$ with $\beta>0$, $\lambda>0$, and
$l\in\R$ (density recalled in Appendix~\ref{apd:gnp}). The parameter $l$ is a
location parameter,~$\lambda$ controls dispersion, and $\beta$ controls tail
decay; $\beta=2$ corresponds to the Gaussian case and $\beta=1$ to the Laplace
case.

For any $x\in\X$ and parameters $\beta$, $l$, $\lambda$, we define the
predictive CDF
\begin{equation}
\label{eq:bcrgp-cdf}
\hat F_n^{\beta, l, \lambda}(z\mid x)
\;=\;
\Theta_{\beta, l, \lambda}\left(
  \frac{z - f_n(x)}{\sigma_n(x)}
\right),
\qquad z\in\R,
\end{equation}
with the convention
$\hat F_n^{\beta,l,\lambda}(z\mid x)=\mathbf{1}\{z\ge f_n(x)\}$ when
$\sigma_n(x)=0$, where $\Theta_{\beta, l, \lambda}$ is the CDF of
$\mathcal{GN}(\beta, l,\lambda)$. With noiseless evaluations and a strictly
positive definite kernel,~$\{x \in \mathbb{X} : \sigma_n(x) = 0\} = \{X_1,
\dots, X_n\}$, so this convention concerns only the observed points.

\citet{pion:2025} fix $l=0$ and select $(\beta,\lambda)$ via a
Bayesian procedure inspired by tolerance-interval constructions
\citep{meeker2017statistical}, targeting $\mu$-probabilistic
calibration of the predictive distribution over $\X$.  The resulting
method is the \emph{Bayesian calibration of residuals for GPs} (\bcrgp).

\section{Spatial Calibration below a Threshold}
\label{sec:calibr-below}

\subsection{Statistical Formulation}
\label{sec:stat-form}

In this section, we introduce notions of calibration below a threshold
$t\in\R$ with respect to a fixed measure $\mu$ on $\X$. 
For a given $n$, consider predictive CDFs $\hat F_n(\cdot\mid x)$,
$x\in\X$, such that for $\mu$-almost all $x$ the map
$z\mapsto \hat F_n(z\mid x)$ is continuous and strictly increasing on
$\R$. 
We assess calibration of the predictive distributions in the lower
tail up to~$t$.

For $z\le t$, the predictive CDF can be decomposed as
\begin{equation}
\hat F_n(z\mid x)=\hat F_n(t\mid x)\,\hat F_{n,t}(z\mid x),
\end{equation}
where $\hat F_n(t\mid x)$ is the predicted probability mass below~$t$, and
$\hat F_{n,t}(\cdot\mid x)$ is the predictive CDF truncated below~$t$,
such that, for all $z\le t$,
\begin{equation}
\hat F_{n,t}(z\mid x)=\P_{\hat F_n(\cdot\mid x)}(Z\le z\mid Z\le t)
=\frac{\hat F_n(z\mid x)}{\hat F_n(t\mid x)},
\end{equation}
with $Z\sim \hat F_n(\cdot\mid x)$. We therefore introduce two
complementary notions: \emph{thresholded $\mu$-calibration}, targeting the
shape of the predictive distribution within the thresholded region through
$\hat F_{n,t}$, and \emph{occurrence calibration}, targeting the probability
mass assigned to the region through $\hat F_n(t\mid x)$ averaged over $\X$.

This decomposition parallels the severity--occurrence
split in probabilistic tail calibration \citep{allen:2025}. In that
setting, \emph{severity} refers to calibration of the conditional
distribution given $Z\le t$, and \emph{occurrence} refers to
calibration of the predictive probability of the event $\{Z\le t\}$. The
definitions introduced here adapt this split to a spatial setting over
$\X$, with inputs drawn from the reference measure $\mu$, and a fixed
threshold~$t$.

\paragraph{Thresholded $\mu$-calibration.}
Thresholded $\mu$-calibration is a tail version of $\mu$-probabilistic
calibration \cite{pion:2025}, obtained by restricting attention to the
thresholded region $\{f(X)\le t\}$ for an auxiliary draw $X\sim\mu$.

Recall the probability integral transform (PIT), which motivates the definition
below: if a random variable $Z$ has a continuous and strictly increasing CDF
$F$, then the PIT~$F(Z)$ is uniformly distributed on $[0,1]$. Hence, if a predictive CDF
$\hat F$ matches the distribution of $Z$, then the PIT~$\hat F(Z)$ is uniform:
departures from uniformity indicate miscalibration. \citet{pion:2025} extend
this idea to a spatial (design-marginal) setting by introducing the $\mu$-PIT,
based on an auxiliary draw $X\sim\mu$. Here we apply the same idea
conditionally on $\{f(X)\le t\}$ through the truncated CDF
$\hat F_{n,t}(\cdot\mid X)$.

Assume that $\mu(\{x\in\X:~f(x)\le t\})>0$ and
$\mu(\{x\in\X:~f(x)=t\})=0$. The second condition avoids an atom at one in the
thresholded PIT. 
The $\mu$--probability integral transform restricted to the thresholded region
($\mu$-tPIT) is defined by
\begin{equation}
U_{t}
:=
\begin{cases}
\hat F_{n,t}(f(X)\mid X), & f(X)\le t,\\[2pt]
1, & f(X)>t.
\end{cases}
\end{equation}
We say that $\hat F_n$ is
\emph{$\mu$-probabilistically calibrated below~$t$} if
$U_t\mid\{f(X)\le t\}$ is uniform on $[0,1]$. This is the thresholded analogue
of $\mu$--probabilistic calibration \citep[][]{pion:2025}: among points such
that $f(X)\le t$, the values $\hat F_{n,t}(f(X)\mid X)$ should be uniformly
distributed on~$[0,1]$.

For $u\in[0,1]$, define
\begin{equation}
\label{eq:G-mu-t}
G_{\mu, t}(u) := \P_n\left(U_{t} \le u \mid f(X)\le t\right). 
\end{equation}
Then $\mu$-probabilistic calibration below~$t$ is equivalent
to~$G_{\mu, t}(u)=u$ for all $u\in[0,1]$.

\paragraph{Occurrence calibration.}
Thresholded $\mu$-calibration constrains the distribution of
$\hat F_{n,t}(f(X)\mid X)$ conditional on~$\{f(X)\le t\}$. It does not constrain
the total probability mass assigned to $\{f(X)\le t\}$, which is governed by
$\hat F_n(t\mid x)$.

Let $p_t$ denote the excursion probability for $X\sim\mu$,
\begin{equation}
\label{eq:pt}
p_t
=
\P\left(f(X)\le t\right).
\end{equation}
The \emph{occurrence discrepancy} is
\begin{equation}
r_{t,n}
=
\Bigl| p_t - \E_n\bigl[\hat F_n(t \mid X)\bigr] \Bigr|,
\qquad X\sim\mu.
\end{equation}
Equivalently, $p_t=\mu(\{x\in\X: f(x)\le t\})$ and
$\E_n[\hat F_n(t\mid X)]=\int_{\X}\hat F_n(t\mid x)\,\mu(\mathrm{d}x)$. We say
that $\hat F_n$ is \emph{occurrence-calibrated at threshold $t$} if
$r_{t,n}=0$.

\subsection{Metrics for Prediction below a Threshold}
\label{sec:metrics_below_threshold}

We assess thresholded $\mu$--probabilistic calibration by comparing the
distribution of the $\mu$--tPIT to the uniform distribution using the
Kolmogorov--Smirnov distance. 
This yields the \emph{tKS--PIT metric}:
\begin{equation}
J_{\mathrm{tKS\text{-}PIT}}(\hat F_n \mid t)
:=
\sup_{u\in[0,1]}
\left|G_{\mu,t}(u) - u\right|.
\end{equation}

\paragraph{Weighted LOO $\mu$-tPIT and tKS--PIT.}
The tKS--PIT and occurrence discrepancy $r_{t,n}$ are defined with respect
to the reference measure $\mu$ on $\X$. In principle, they should be
evaluated on an independent dataset $(X'_j,f(X'_j))_{j=1}^L$
with~$X'_j\sim\mu$. In BO, such an evaluation set is typically
unavailable. We therefore rely on the observed BO data $\Dcal_n$ and
use leave-one-out (LOO) predictive CDFs. Since the BO design points
are not distributed according to $\mu$ in general, we also allow for
reweighting to approximate $\mu$--averages.

Let $\Dcal_n=\{(X_i,Z_i)\}_{i=1}^n$ with $Z_i=f(X_i)$. For
each~$i\in\{1,\ldots,n\}$, let $\hat F_{n,-i}(\cdot\mid x)$ denote the
predictive CDF constructed from $\Dcal_n\setminus\{(X_i,Z_i)\}$. Fix
$t\in\R$.

Let~$(w_i)_{i=1}^n$ be nonnegative weights with $\sum_{j=1}^n w_j>0$ and set
$\tilde w_i=w_i/\sum_{j=1}^n w_j$. The weights are chosen so that
weighted averages over $(X_i)$ approximate $\mu$--averages. The
unweighted case corresponds to $w_i=1$. In this work, $(w_i)$ is
obtained by density-ratio estimation based on a density estimate of
the BO design distribution; details are given in
Appendix~\ref{ap:optimization_algo} (paragraph \emph{Density-ratio
  estimation and weights}).

Assume that $\hat F_{n,-i}(t\mid X_i)>0$ whenever $Z_i\le t$.
Define the LOO empirical $\mu$-tPIT by
\begin{equation}
\label{eq:loo_tpit}
U^{\mathrm{LOO}}_{t,i}
=
\begin{cases}
\hat F_{n,-i}(Z_i\mid X_i)\big/\hat F_{n,-i}(t\mid X_i), & Z_i\le t,\\[2pt]
1, & Z_i>t.
\end{cases}
\end{equation}

We estimate $p_t$ using the weighted empirical frequency
\begin{equation}
\hat p^{\,w}_{t,n}
=
\sum_{i=1}^n \tilde w_i\,\mathbf{1}\{Z_i\le t\}.
\end{equation}
Assume that $\hat p^{\,w}_{t,n}>0$ and define, for all $u\in[0,1]$,
\begin{equation}
\label{eq:G-LOO-w}
G^{\mathrm{LOO},w}_{t,n}(u)
=
\frac{
\sum_{i=1}^n \tilde w_i\,\mathbf{1}\{Z_i\le t\}\,\mathbf{1}\{U^{\mathrm{LOO}}_{t,i}\le u\}
}{
\hat p^{\,w}_{t,n}
}.
\end{equation}
The weighted empirical tKS--PIT metric is
\begin{equation}
J^{\mathrm{LOO},w}_{\mathrm{tKS\text{-}PIT},n}(\hat F_n\mid t)
=
\sup_{u\in[0,1]} \left| G^{\mathrm{LOO},w}_{t,n}(u) - u \right|.
\end{equation}

\paragraph{Weighted estimation of $p_t$ and occurrence discrepancy.}
The corresponding weighted LOO occurrence discrepancy estimator is
\begin{equation}
r^{\mathrm{LOO},w}_{t,n}
=
\biggl|
\hat p^{\,w}_{t,n}
-
\sum_{i=1}^n \tilde w_i\,\hat F_{n,-i}(t\mid X_i)
\biggr|.
\end{equation}

When an independent test set is available (as in our synthetic
benchmarks), we also report direct test-set versions of these metrics;
see Appendix~\ref{ap:metrics_testset}.

\begin{remark}
  Calibration alone does not guarantee useful uncertainty
  quantification. A well-calibrated predictive distribution can be
  overly diffuse. Proper scoring rules assess calibration and
  sharpness jointly (see Appendix~\ref{ap:scoring_rules}).
\end{remark}

\subsection{Calibration below a Threshold in BO}
\label{sec:impact_calibration_bo}

At iteration $n$, EI depends on the predictive lower tail below the
current best value $m_n$, through $\hat F_n(m_n\mid x)$ and the
restriction of $\hat F_n(\cdot\mid x)$ to $(-\infty,m_n]$. However, calibration 
at level $m_n$ cannot be assessed from $\Dcal_n$: 
any observation with $Z_i \le m_n$ satisfies $Z_i=m_n$, so every LOO~$\mu$-tPIT
value~\eqref{eq:loo_tpit} equals one. The occurrence discrepancy remains
computable at $t=m_n$, but $\hat p^{w}_{m_n,n}$ then rests on a single
observation. We therefore impose
calibration below a higher threshold $t_n$, chosen from the empirical
$\delta$-quantile~$q_{\delta,n}$ of $(f(X_1),\ldots,f(X_n))$ for some
$\delta\in(0,1]$. In the BO algorithm, this threshold can be kept fixed when
too few observations fall below the new quantile.

Appendix~\ref{ap:impact_calibration_bo} gives a formal consistency
statement showing that when the set of global minimizers is $\mu$-negligible,
occurrence calibration at $m_n$ forces the average predicted improvement
probability to vanish as $m_n$ approaches the optimum.

\section{A Goal-Oriented Calibration Method}
\label{sec:goal-orient-calibr}

\subsection{Description of the Method}
\label{sec:selection_parameters}

In this section, we introduce a procedure to select the parameters
$\beta$ and $\lambda$ of the GN predictive model.
We target thresholded $\mu$-calibration below a fixed threshold $t$,
assessed via the truncated CDF $\hat F_{n,t}$ on $\{f(X)\le t\}$, and
occurrence calibration at $t$, assessed by matching
$p_t$ with $\E_n \bigl[\hat F_n(t\mid X)\bigr]$.
Following \citet{pion:2025}, we fix $l=0$, so that calibration changes the
shape of the predictive distribution without shifting its mean away from the
GP mean $f_n$. Accordingly, we write $\hat F_n^{\beta,\lambda}(\cdot\mid x)$
instead of $\hat F_n^{\beta,0,\lambda}(\cdot\mid x)$.

\paragraph{Parameter selection criterion.}
Let $U_t^{\beta,\lambda}$ be the $\mu$-tPIT associated with the
truncated predictive CDF $\hat F_{n,t}^{\beta,\lambda}(\cdot\mid
X)$. Thresholded $\mu$--probabilistic calibration below $t$ is
equivalent to
\begin{equation}
\P_n\left(U_t^{\beta,\lambda} \le u \,\middle|\, f(X)\le t\right)=u,
\qquad u\in[0,1].
\end{equation}

To incorporate occurrence calibration at $t$, define the occurrence
ratio
\begin{equation}
\kappa_t^{\beta,\lambda}
=
\frac{\E_n\left[\hat F_n^{\beta,\lambda}(t\mid X)\right]}
{p_{t}}.
\end{equation}
The numerator is the $\mu$-average predicted probability
of~$\{f(X)\le t\}$, while the denominator is the $\mu$-probability
of~$\{f(X)\le t\}$, so $\kappa_t^{\beta,\lambda}=1$ is equivalent
to occurrence calibration at $t$. Note that $\kappa_t^{\beta,\lambda}>1$
indicates an overestimation of $p_t$ and $\kappa_t^{\beta,\lambda}<1$ an
underestimation.

We compare $u\mapsto \P_n(U_t^{\beta,\lambda} \le u \mid f(X)\le t)$ to
$u\mapsto u\,\kappa_t^{\beta,\lambda}$ and define
\begin{equation}
\label{eq:objective_tcgp}
J(\beta,\lambda)
=
\sup_{u\in[0,1]}
\left|
\P_n\left(U_t^{\beta,\lambda} \le u \,\middle|\, f(X)\le t\right)
-
u\,\kappa_t^{\beta,\lambda}
\right|.
\end{equation}
By construction, $J(\beta,\lambda)=0$ when $\hat F_n^{\beta,\lambda}$
is $\mu$--probabilistically calibrated below $t$ and
occurrence-calibrated at $t$.  When
$\kappa_t^{\beta,\lambda}=1$, $J(\beta,\lambda)$ reduces to the
tKS--PIT metric. Appendix~\ref{ap:objective} reports alternative
parameter selection criteria.  We use $J$ in the experiments.

\paragraph{LOO approximation.}
The criterion $J(\beta,\lambda)$ involves $\mu$--expectations and the
excursion probability $p_t$, which are unknown given
only $\Dcal_n$. We use a weighted leave-one-out (LOO) approximation. 
Let $G^{\mathrm{LOO},w}_{t,n}(u;\beta,\lambda)$ be defined by
\eqref{eq:G-LOO-w}, with $\hat F_{n,-i}(\cdot\mid x)$ replaced by
$\hat F_{n,-i}^{\beta,\lambda}(\cdot\mid x)$.  Define
\begin{equation}
\widehat\kappa^{\beta,\lambda}_{t,n}
=
\frac{
\sum_{i=1}^n \tilde w_i\,\hat F_{n,-i}^{\beta,\lambda}(t\mid X_i)
}{
\hat p^{\,w}_{t,n}
},
\end{equation}
assuming $\hat p^{\,w}_{t,n}>0$. The resulting approximation of
\eqref{eq:objective_tcgp} is
\begin{equation}
J^{\mathrm{LOO},w}_{t,n}(\beta,\lambda)
=
\sup_{u\in[0,1]}
\left|
G^{\mathrm{LOO},w}_{t,n}(u;\beta,\lambda)
-
u\,\widehat\kappa^{\beta,\lambda}_{t,n}
\right|.
\end{equation}

The parameters $\beta$ and $\lambda$ are selected by solving
\begin{equation}
\label{eq:tail-calibrated-gp}
(\beta^*,\lambda^*)
=
\argmin_{(\beta,\lambda)\in[\beta_0,\beta_1]\times[\lambda_0,\lambda_1]}
J^{\mathrm{LOO},w}_{t,n}(\beta,\lambda),
\end{equation}
where $0<\lambda_0<\lambda_1$ and $0<\beta_0<\beta_1$. We refer to the method as
\emph{calibrated prediction below~$t$ for GPs} (\tcgp).

The complete \tcgp\ procedure, including its integration into BO, is described
in Appendix~\ref{ap:optimization_algo}.

\subsection{EI with \tcgp}
\label{sec:opt_calibrated}

Section~\ref{sec:bg-setting} defined EI and the BO loop. We now describe the EI
algorithm obtained by using \tcgp predictive distributions.

\paragraph{Choice of the threshold.}
The objective is to obtain predictive distributions that are well
calibrated below the current best value $m_n$. As discussed in
Section~\ref{sec:impact_calibration_bo}, calibration at level $m_n$
cannot be assessed from $\Dcal_n$. We therefore calibrate below a higher 
threshold $t_n$ chosen in the lower tail of the observed responses.
Let $q_{\delta,n}$ be the empirical~$\delta$-quantile of
$(Z_1,\ldots,Z_n)$. At the initial iteration, we set
$t_{n_0}=q_{\delta,n_0}$. We use
$q_{\delta,n}$ as the candidate threshold
when~$\hat p^{\,w}_{q_{\delta,n},n}\ge p_{\min}$, in which case
$t_n=q_{\delta,n}$. If not, we keep the previous threshold. Here,
$p_{\min}$ is the minimum estimated excursion probability below a candidate
threshold that we allow. This avoids calibrating on too few points when the
design concentrates near a minimizer and the lower-tail excursion probability
becomes small.

\paragraph{Predictive distribution and EI.}
Fix $(\beta,\lambda)$. Conditional on $\Dcal_n$ and for $x\in\X$,
\tcgp uses the predictive CDF~$\hat F_n^{\beta,\lambda}(\cdot\mid x)$
defined in~\eqref{eq:bcrgp-cdf}.  Let $Z_{n,x}$ be a predictive random
variable with this CDF. With the GN model,
\begin{equation}
Z_{n,x} \sim \mathcal{GN}\bigl(\beta, f_n(x), \lambda\,\sigma_n(x)\bigr),
\end{equation}
with the convention that $Z_{n,x}=f_n(x)$ when $\sigma_n(x)=0$.

The EI criterion induced by $\hat F_n^{\beta,\lambda}(\cdot\mid x)$ is
\begin{equation}
\rho_n(x)
=
\E\bigl[(m_n-Z_{n,x})_+\mid \Dcal_n\bigr],
\end{equation}
where the expectation is taken with respect to the predictive
distribution of $Z_{n,x}$ conditional on
$\Dcal_n$. Proposition~\ref{prop:ei_gn} yields the closed form
\begin{equation}
\rho_n(x)=\gamma\bigl(m_n-f_n(x), \lambda\,\sigma_n(x), \beta\bigr),
\end{equation}
where $\gamma$ is defined in~\eqref{eq:ei_gn}.

\begin{proposition}
\label{prop:ei_gn}
If $Z \sim \mathcal{GN}(\beta,l,\lambda)$, then for any $a\in\R$,
\begin{equation}
\E\bigl[(a - Z)_+\bigr]
=
\gamma(a-l, \lambda, \beta),
\end{equation}
where, for $\lambda>0$,
\begin{equation}
\label{eq:ei_gn}
\gamma(z, \lambda, \beta)
=
z\, \Theta_{\beta}\left(\frac{z}{\lambda}\right)
+ \frac{\lambda}{2\, \Gamma(1/\beta)}
\Gamma\left(\frac{2}{\beta}, \left|\frac{z}{\lambda}\right|^{\beta}\right),
\end{equation}
and $\gamma(z,0,\beta)=\max(z,0)$. Here $\Theta_\beta$ denotes the CDF of
$\mathcal{GN}(\beta,0,1)$ and $\Gamma(\cdot,\cdot)$ is the upper
incomplete gamma function. For $\lambda>0$,
$\gamma(\cdot,\lambda,\beta)$ is continuous and
satisfies~$\gamma(z,\lambda,\beta)>0$ for all $z\in\R$.
\end{proposition}
\begin{proof}
See Appendix~\ref{ap:ei_gn}.
\end{proof}

\begin{remark}
  For $\beta=2$, the GN distribution coincides with a
  Gaussian distribution, and~\eqref{eq:ei_gn} reduces to the standard
  EI expression; see Appendix~\ref{ap:connection_gaussian_ei}.
\end{remark}

\paragraph{BO update.}
At iteration $n$, we set $t=t_n$ and obtain~$(\beta_n,\lambda_n)$ by solving
\eqref{eq:tail-calibrated-gp}. The next evaluation point is selected as
\begin{equation}
X_{n+1}\in\arg\max_{x\in\X}\rho_n(x).
\end{equation}
Additional implementation details are given in
Appendix~\ref{ap:optimization_algo}.

\subsection{Convergence of EI with \tcgp and Fixed GP Hyperparameters}
\label{sec:convergence}
\vskip 1em

We prove a convergence result for EI when predictive distributions are obtained
with \tcgp from a zero-mean GP model with fixed covariance kernel $k$. This is a
sanity check: with fixed hyperparameters and bounded \tcgp parameters, EI still
produces a dense sequence of evaluation points in $\X$.
Proposition~\ref{prop:convergence} extends the exploration result of
\citet{vazquez:2010} by replacing the Gaussian predictive distribution with the
GN family, with parameters $(\beta_n,\lambda_n)$ constrained to
a compact set.

We assume that $k$ satisfies the no-empty ball (NEB) property \citep{vazquez:2010}:
for any sequence $(X_n)_{n\ge 1}$ in $\X$ and any $x\in\X$, $x$ is an adherent
point of $\{X_n:n\ge 1\}$ iff $\sigma_n^2(x)\to 0$ as $n\to\infty$.

\vskip 1em
\begin{proposition}
\label{prop:convergence}
Assume that $\X$ is compact and that $k$ is continuous, stationary, strictly
positive definite, and satisfies the NEB property. Let $\mathcal H$ be the
associated RKHS and assume $f\in\mathcal H$. Let $f_n$ and $\sigma_n$ be the
zero-mean noiseless kernel interpolant and kriging standard deviation
constructed from $\Dcal_n$. Assume that the design points in the initial
design~$\Dcal_{n_0}$ are pairwise distinct and that there exists $M\ge 1$ such that
for all $n\ge M$,
$$
(\beta_n,\lambda_n)\in[\beta_0,\beta_1]\times[\lambda_0,\lambda_1],
$$
with $0<\beta_0<\beta_1$ and $0<\lambda_0<\lambda_1$. Let $(X_n)_{n\ge 1}$ be
generated by EI, that is
$$
X_{n+1}\in\argmax_{x\in\X}\rho_n(x),
$$
where $\rho_n$ is EI computed from the \tcgp predictive CDFs with parameters
$(\beta_n,\lambda_n)$. Then $(X_n)$ is dense in $\X$.
\end{proposition}

The maximizer exists because $f_n$, $\sigma_n$, and $\gamma$ are continuous, so
$\rho_n$ is continuous on the compact set $\X$.

\begin{proof}
See Appendix~\ref{ap:convergence}.
\end{proof}

\vskip 1em

Proposition~\ref{prop:convergence} implies that $(X_n)$ intersects every
nonempty open subset of $\X$. Since $f$ is continuous on the compact set $\X$,
for any global minimizer $x^*\in\argmin_{x\in\X} f(x)$ there exists a
subsequence $(X_{n_j})_j$ such that $X_{n_j}\to x^*$,
hence~$m_n \to \min_{x\in\X} f(x)$.

\vskip 1em

\begin{remark}
  In practice, GP hyperparameters are re-estimated at each iteration. Extending
  Proposition~\ref{prop:convergence} to sequential hyperparameter selection
  requires a separate analysis. We leave this extension to future work.
\end{remark}

\newpage
\section{Experiments}
\label{sec:experiment_calibration}

\subsection{Experimental Setup}
\label{sec:experimental-set-up}

We study how calibration below a threshold affects BO. We report
results for variants targeting occurrence calibration, thresholded 
$\mu$-calibration,
or both, and we compare \tcgp\ with \regp\ (relaxed interpolation above $t_n$),
\bcrgp\ with density-ratio reweighting (Appendix~\ref{ap:bcrgp_reweight}) for
calibration over $\X$, and the online conformal recalibration method of
\citet{deshpande:2024} based on the quantile pinball loss (\onego). For \regp,
we use $\delta=0.25$, and for \tcgp, $\delta=~0.05$ (see also Appendices~\ref
{ap:delta_on_regp} and~\ref{ap:delta_on_tcgp}). We fix $p_{\min}=0.015$.
Additional diagnostics at $t_n$ and at the current best value $m_n$ are reported
in Appendix~\ref{ap:fixed_n_exp}.

All methods use a GP model with constant mean and an anisotropic
Mat\'ern kernel (Appendix~\ref{ap:parameters}), with hyperparameters
selected by maximum likelihood at each iteration. Experiments are
implemented with \texttt{gpmp} \citep{gpmp_2026}. The benchmark uses standard
deterministic test functions (Appendix~\ref{ap:test_functions}). Each
run starts from an initial design of size~$n_0=10d$ drawn uniformly on
$\X$, followed by EI-based BO iterations. EI is maximized using a
sequential Monte Carlo procedure available in \texttt{gpmp}
(Appendix~\ref{ap:max_ei}). For each test function, we
generate $100$ independent initial datasets.

BO performance is summarized by the excursion probability below the
current best value $m_n$: for each run and each iteration $n$, we
consider $p_{m_n} = \P\left(f(X)\le m_n\right)$,
$X\sim~\mathcal{U}(\X)$, and estimate $p_{m_{n}}$ by subset simulation
\citep{bect2017bayesian} as the main performance metric. We then
report, as a function of $n$, the median and the $10\%$ and $90\%$
quantiles of $p_{m_n}$ across the $100$
runs.

\begin{figure*}[htb]
    \centering
     \begin{subfigure}[b]{\textwidth}
        \centering
        \includegraphics[width=\textwidth]{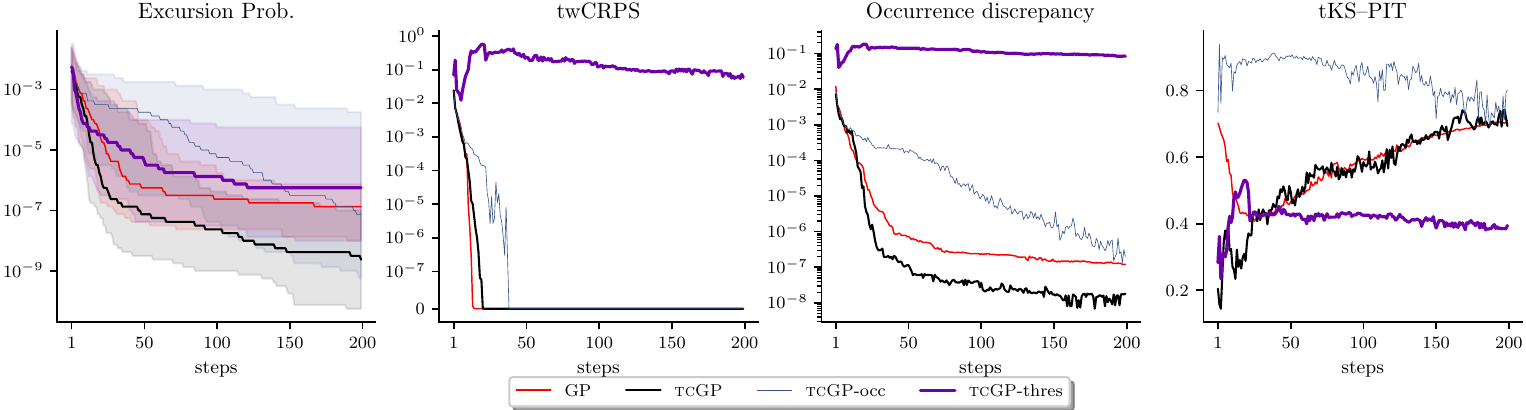}
        \caption{Ackley with $d=4$}
    \end{subfigure}
    \hfill
    \begin{subfigure}[b]{\textwidth}
        \centering
        \includegraphics[width=\textwidth]{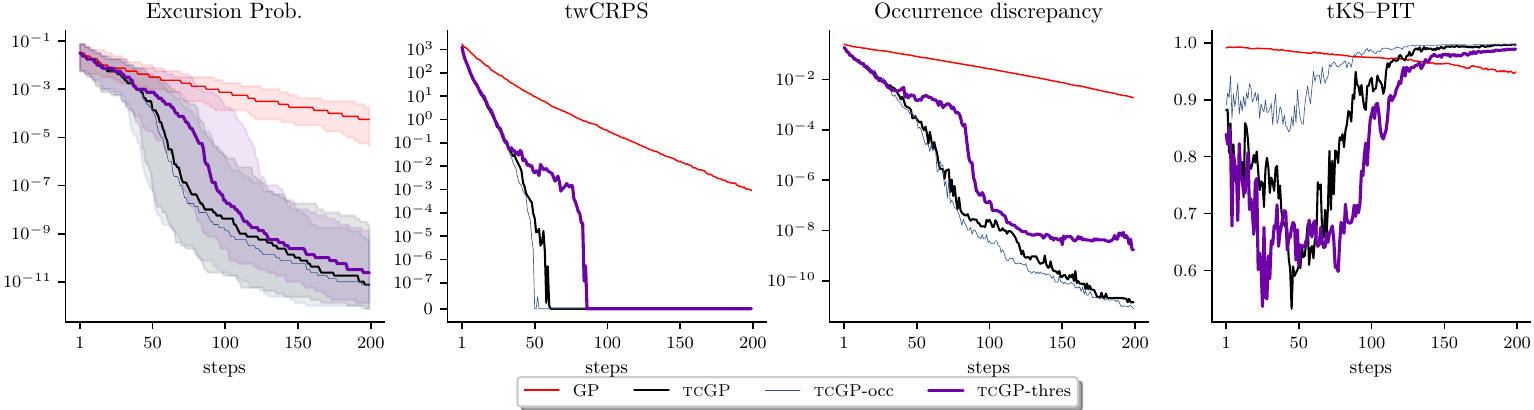}
        \caption{Goldstein--Price}
      \end{subfigure}
      \caption{BO performance and calibration metrics.  From left to
        right: median and $10\%/90\%$ quantiles across runs of the
        estimated excursion probability $p_{m_n}=\P(f(X)\le m_n)$ with
        $X\sim\mathcal{U}(\X)$; median twCRPS; median occurrence
        discrepancy $r_{m_n,n}$; and median tKS--PIT. Calibration metrics
        are evaluated on a test set at the current best value $m_n$.
        Results are shown for a standard GP and three \tcgp variants
        using $J$, tKS--PIT, or $r^{\mathrm{LOO},w}_{t,n}$ as the parameter selection
        criterion ($\delta=0.05$).}
    \label{fig:calibration}
\end{figure*}

\begin{figure*}[h!]
  \centering
  \includegraphics[width=\textwidth]{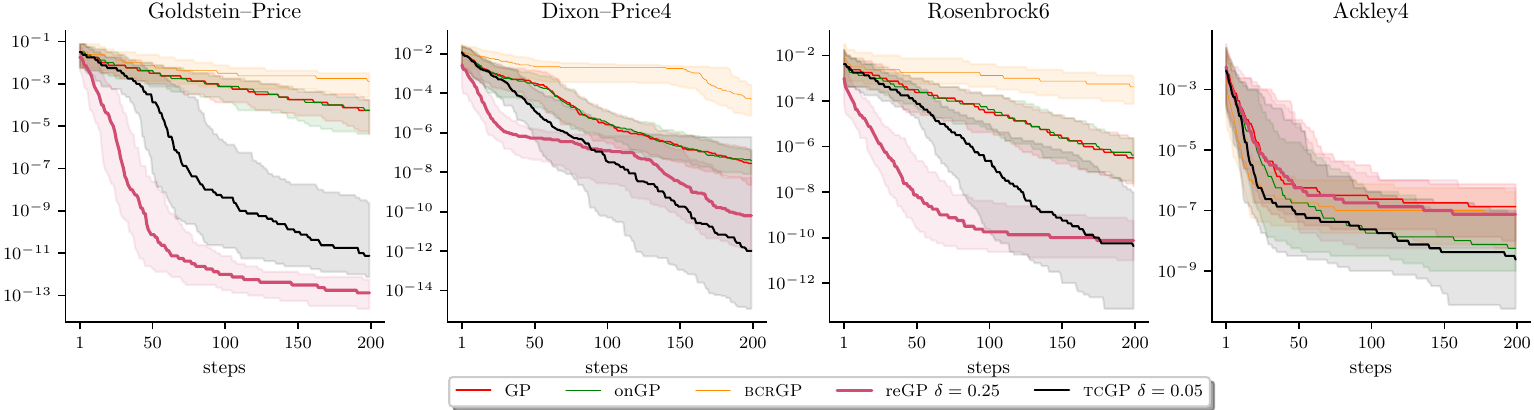}
      \caption{BO performance summarized by the excursion probability below the current best value.
      For each run and iteration $n$, we estimate~$p_{m_n}=\P(f(X)\le m_n)$ with $X\sim\mathcal{U}(\X)$,
      and report the median and $10\%/90\%$ quantiles of $p_{m_n}$ across runs for Goldstein--Price,
      Dixon--Price ($d=4$), Rosenbrock ($d=6$), and Ackley ($d=4$).
      Methods: GP, \bcrgp, \onego, \regp\ ($\delta=0.25$), \tcgp\ ($\delta=0.05$).}
    \label{fig:combined}
\end{figure*}

\subsection{Thresholded and Occurrence Calibration: Empirical Comparison}
\label{sec:comp-betw-calibr}

We compare the roles of occurrence calibration, thresholded tail-rank
calibration, and their combination in BO. We compare a standard GP model
with three \tcgp variants. The first variant, \tcgploc, selects
$(\beta,\lambda)$ by minimizing the weighted LOO tKS--PIT, targeting
thresholded $\mu$--calibration. The second, \tcgpmar, minimizes the weighted LOO
occurrence discrepancy $r^{\mathrm{LOO},w}_{t,n}$. The joint variant, \tcgp, uses the selection
criterion of Section~\ref{sec:goal-orient-calibr}.

At each BO iteration, we report BO performance and test-set
calibration metrics evaluated at the current best value $m_n$: twCRPS,
tKS--PIT below $m_n$, and $r_{m_n,n}$ (details about the metrics provided
in Appendices~\ref{app:twscrps} and~\ref{ap:metrics_testset}). The
twCRPS and $r_{m_n,n}$ use $1000$ test points sampled in $\X$, while
tKS--PIT uses $900$ test points conditioned on $f(x)\le m_n$.

Figure~\ref{fig:calibration} reports results for Ackley ($d=4$) and
Goldstein--Price. On Ackley, \tcgploc\ gives the lowest tKS--PIT after the
early iterations, but its twCRPS and occurrence discrepancy are much larger and
its median $p_{m_n}$ decreases more slowly. \tcgpmar\ gives the largest
tKS--PIT, and its occurrence discrepancy exceeds that of the GP over most of the
run. Its median $p_{m_n}$ is close to the GP curve by the end. The joint method
\tcgp\ gives the lowest median $p_{m_n}$ and occurrence discrepancy. Its
tKS--PIT is lower than that of the GP at the beginning and similar later.

On Goldstein--Price, all \tcgp-based variants improve over the GP and quickly
reduce twCRPS and $r_{m_n,n}$. \tcgp\ and \tcgpmar\ behave similarly and outperform
\tcgploc, again indicating that occurrence calibration has a stronger
effect on EI-driven BO than thresholded $\mu$-calibration alone.

Overall, these results suggest that accurate occurrence calibration at $m_n$
is often important for BO, with thresholded~$\mu$-calibration providing an
additional gain. Focusing only on the thresholded notion can yield poor 
optimization performance because it does not control systematic over- or
under-estimation of improvement probabilities over $\X$.

\vspace{1\baselineskip}

\subsection{Comparison across Methods}
\label{sec:comp-mutl-meth}

We compare EI-based BO with a standard GP model, \regp, \tcgp, \onego, and
\bcrgp\ on four test functions: Ackley ($d=4$), Dixon--Price ($d=4$),
Rosenbrock ($d=6$), and Goldstein--Price. Figure~\ref{fig:combined} reports the
evolution of the median and the $10$\%/$90$\% quantiles of
$p_{m_n} = \P(f(X) \le m_n)$, where $m_n$ denotes the best observed value so 
far. To complement these results, Appendix~\ref{app:add_results_samp_crit}
reports additional EI experiments on Ackley, Cross-In-Tray, Dixon--Price,
Goldstein--Price, Hartmann, Michalewicz, Rosenbrock, Shekel, and Perm, in
dimensions ranging from $d=2$ to $d=20$, together with additional results for 
the \emph{upper confidence bound} (UCB) criterion.

On Goldstein--Price, Dixon--Price, and Rosenbrock, \regp\ and \tcgp\ improve 
over the standard GP model, with \regp\ performing best on these smoother
objectives. On Ackley, \tcgp\ provides the largest gains.

\bcrgp\ does not consistently improve over the standard GP model, suggesting
that global calibration over $\X$ alone is not sufficient to improve EI. The
online conformal method \onego\ is close to the standard GP model on these
noiseless benchmarks, except on Ackley.

Appendix~\ref{app:add_results_samp_crit} reports similar qualitative patterns
for both EI and UCB sampling criteria. On functions where the standard GP model
already performs well, such as Michalewicz or Shekel, \tcgp\ and \regp\ provide
only small improvements. In these cases, calibration has little effect on BO
performance.

Calibrating below a low threshold improves BO in these settings. \regp\ and
\tcgp\ perform best on different test functions, with \regp\ incurring a larger
computational overhead than \tcgp\ (Appendix~\ref{sec:computation-time}).

\vspace{1\baselineskip}

\section{Discussion and Limitations}
\label{sec:discussion}

We introduced a goal-oriented calibration framework for Bayesian optimization,
built on two design-marginal notions: occurrence calibration at a threshold and
thresholded~$\mu$-calibration below that threshold. The ablation study indicates
that the joint criterion is more consistent than variants targeting only one of
the two notions. In the deterministic GP-based experiments, \tcgp\ changes the
behavior of EI and UCB and improves performance on several benchmarks, while
requiring less computation for model construction than \regp.

The study is restricted to the exact-evaluation setting. In noisy problems, the
latent objective is no longer directly observed, so the calibration targets
considered here are not directly accessible from the data and would need to be
redefined. The experiments are also mostly low-to-moderate
dimensional, where the weighted leave-one-out estimates are more stable. Future
work can study extensions to noisy BO and to higher-dimensional problems.

\FloatBarrier

\vskip 1cm
\section*{Acknowledgements}

This work was supported by \textsc{transvalor s.a.} through a CIFRE PhD 
agreement. We thank \textsc{transvalor s.a.} for its support.

\section*{Impact Statement}
Improved lower-tail calibration can reduce the number of expensive evaluations
required in Bayesian optimization. The impact is application-dependent, since
the same methodology can be used for beneficial or harmful optimization goals.

\section*{Conflict of Interest Disclosure}
Aurélien Pion is affiliated with \textsc{transvalor s.a.} The article does not
evaluate a proprietary product or model developed by \textsc{transvalor s.a.}
The authors declare no other financial conflicts of interest.

\newpage
\bibliography{bib}

\begin{thebibliography}{44}
\providecommand{\natexlab}[1]{#1}
\providecommand{\url}[1]{\texttt{#1}}
\expandafter\ifx\csname urlstyle\endcsname\relax
  \providecommand{\doi}[1]{doi: #1}\else
  \providecommand{\doi}{doi: \begingroup \urlstyle{rm}\Url}\fi

\bibitem[Allen et~al.(2023)Allen, Bhend, Martius, and Ziegel]{allen:2023}
Allen, S., Bhend, J., Martius, O., and Ziegel, J.
\newblock Weighted verification tools to evaluate univariate and multivariate
  probabilistic forecasts for high-impact weather events.
\newblock \emph{Weather and Forecasting}, 38\penalty0 (3):\penalty0 499 -- 516,
  2023.

\bibitem[Allen et~al.(2025)Allen, Koh, Segers, and Ziegel]{allen:2025}
Allen, S., Koh, J., Segers, J., and Ziegel, J.
\newblock Tail calibration of probabilistic forecasts.
\newblock \emph{J. Amer. Statist. Assoc.}, 120\penalty0 (552):\penalty0
  2796--2808, 2025.

\bibitem[Auer et~al.(2002)Auer, Cesa{-}Bianchi, and Fischer]{auer2002finite}
Auer, P., Cesa{-}Bianchi, N., and Fischer, P.
\newblock Finite-time analysis of the multiarmed bandit problem.
\newblock \emph{Mach. Learn.}, 47:\penalty0 235--256, 2002.

\bibitem[Bect et~al.(2017)Bect, Li, and Vazquez]{bect2017bayesian}
Bect, J., Li, L., and Vazquez, E.
\newblock {B}ayesian subset simulation.
\newblock \emph{SIAM/ASA Journal on Uncertainty Quantification}, 5\penalty0
  (1):\penalty0 762--786, 2017.

\bibitem[Bogunovic \& Krause(2021)Bogunovic and Krause]{Bogunovic:2021}
Bogunovic, I. and Krause, A.
\newblock Misspecified {G}aussian process bandit optimization.
\newblock In Ranzato, M., Beygelzimer, A., Dauphin, Y., Liang, P., and Vaughan,
  J.~W. (eds.), \emph{Adv. Neural Inf. Process. Syst.}, volume~34, pp.\
  3004--3015. Curran Associates, Inc., 2021.

\bibitem[Chu \& Ghahramani(2005)Chu and Ghahramani]{JMLR:v6:chu05a}
Chu, W. and Ghahramani, Z.
\newblock {G}aussian processes for ordinal regression.
\newblock \emph{J. Mach. Learn. Res.}, 6\penalty0 (35):\penalty0 1019--1041,
  2005.

\bibitem[Deshpande et~al.(2024)Deshpande, Marx, and Kuleshov]{deshpande:2024}
Deshpande, S., Marx, C., and Kuleshov, V.
\newblock Online calibrated and conformal prediction improves {B}ayesian
  optimization.
\newblock In Dasgupta, S., Mandt, S., and Li, Y. (eds.), \emph{Proceedings of
  The 27th International Conference on Artificial Intelligence and Statistics},
  volume 238 of \emph{Proc. Mach. Learn. Res.}, pp.\  1450--1458. PMLR, 02--04
  May 2024.

\bibitem[Feliot et~al.(2016)Feliot, Bect, and Vazquez]{Feliot_2016}
Feliot, P., Bect, J., and Vazquez, E.
\newblock A {B}ayesian approach to constrained single- and multi-objective
  optimization.
\newblock \emph{J. Global Optim.}, 67\penalty0 (1-2):\penalty0 97--133, April
  2016.

\bibitem[Forrester et~al.(2008)Forrester, S{\'o}bester, and
  Keane]{forrester:2008}
Forrester, A.~I.~J., S{\'o}bester, A., and Keane, A.~J.
\newblock \emph{Engineering Design via Surrogate Modelling: A Practical Guide}.
\newblock John Wiley \& Sons, 2008.

\bibitem[Friedli et~al.(2025)Friedli, Gautier, Broccard, and
  Ginsbourger]{friedli:2025}
Friedli, L., Gautier, A., Broccard, A., and Ginsbourger, D.
\newblock {CRPS}-based targeted sequential design with application in chemical
  space, 2025.
\newblock URL \url{https://arxiv.org/abs/2503.11250}.

\bibitem[Gneiting \& Resin(2023)Gneiting and Resin]{Gneiting:2023}
Gneiting, T. and Resin, J.
\newblock Regression diagnostics meets forecast evaluation: conditional
  calibration, reliability diagrams, and coefficient of determination.
\newblock \emph{Electron. J. Stat.}, 17\penalty0 (2), January 2023.

\bibitem[Gneiting et~al.(2007)Gneiting, Balabdaoui, and
  Raftery]{GneitingRaftery2007}
Gneiting, T., Balabdaoui, F., and Raftery, A.~E.
\newblock Probabilistic forecasts, calibration and sharpness.
\newblock \emph{J. R. Stat. Soc. Ser. B Stat. Methodol.}, 69\penalty0
  (2):\penalty0 243--268, 2007.

\bibitem[Gneiting \& Raftery(2007)Gneiting and
  Raftery]{gneiting07:scoring_rules}
Gneiting, T.~G. and Raftery, A.~E.
\newblock Strictly proper scoring rules, prediction, and estimation.
\newblock \emph{J. Am. Stat. Assoc.}, 102\penalty0 (477):\penalty0 359--378,
  2007.

\bibitem[Guo et~al.(2021)Guo, Ong, and Liu]{Guo2021CalibratedEI}
Guo, Z., Ong, Y.~S., and Liu, H.
\newblock Calibrated and recalibrated expected improvements for {B}ayesian
  optimization.
\newblock \emph{Struct. Multidiscip. Optim.}, 64:\penalty0 3549--3567, 2021.

\bibitem[Jin \& Ren(2025)Jin and
  Ren]{jin2024confidencefocalconformalprediction}
Jin, Y. and Ren, Z.
\newblock Confidence on the focal: conformal prediction with
  selection-conditional coverage.
\newblock \emph{J. R. Stat. Soc. Ser. B. Stat. Methodol.}, 87\penalty0
  (4):\penalty0 1239--1259, 04 2025.

\bibitem[Jones et~al.(1998)Jones, Schonlau, and Welch]{Jones1998:article_ego}
Jones, D., Schonlau, M., and Welch, W.
\newblock Efficient global optimization of expensive black-box functions.
\newblock \emph{J. Global Optim.}, 13:\penalty0 455--492, 12 1998.

\bibitem[Kim et~al.(2025)Kim, Zecchin, Park, Kang, and Simeone]{kim:2025}
Kim, D., Zecchin, M., Park, S., Kang, J., and Simeone, O.
\newblock Robust {B}ayesian optimization via localized online conformal
  prediction.
\newblock \emph{IEEE Trans. Signal Process.}, 73:\penalty0 2039--2052, 2025.

\bibitem[Lai \& Robbins(1985)Lai and Robbins]{lairobbins1985}
Lai, T.~L. and Robbins, H.
\newblock Asymptotically efficient adaptive allocation rules.
\newblock \emph{Adv. Appl. Math.}, 6\penalty0 (1):\penalty0 4--22, 1985.

\bibitem[Meeker et~al.(2017)Meeker, Hahn, and Escobar]{meeker2017statistical}
Meeker, W.~Q., Hahn, G.~J., and Escobar, L.~A.
\newblock \emph{Statistical Intervals: A Guide for Practitioners and
  Researchers}.
\newblock John Wiley \& Sons, Hoboken, New Jersey, second edition, 2017.
\newblock ISBN 978-0-471-68717-7.

\bibitem[Mockus et~al.(1978)Mockus, Tiesis, and Zilinskas]{mockus:1978}
Mockus, J., Tiesis, V., and Zilinskas, A.
\newblock The application of {B}ayesian methods for seeking the extremum.
\newblock In Dixon, L.~C.~W. and Szeg{\"o}, G.~P. (eds.), \emph{Towards Global
  Optimisation}, volume~2, pp.\  117--129. North-Holland, Amsterdam, 1978.

\bibitem[Nadarajah(2005)]{nadarajah2005}
Nadarajah, S.
\newblock A generalized normal distribution.
\newblock \emph{J. Appl. Stat.}, 32\penalty0 (7):\penalty0 685--694, 2005.

\bibitem[Neiswanger \& Ramdas(2021)Neiswanger and Ramdas]{Neiswanger:2021}
Neiswanger, W. and Ramdas, A.
\newblock Uncertainty quantification using martingales for misspecified
  {G}aussian processes.
\newblock In Feldman, V., Ligett, K., and Sabato, S. (eds.), \emph{Proceedings
  of the 32nd International Conference on Algorithmic Learning Theory}, volume
  132 of \emph{Proc. Mach. Learn. Res.}, pp.\  963--982. PMLR, 16--19 Mar 2021.

\bibitem[Nelder \& Mead(1965)Nelder and Mead]{nelder_mead}
Nelder, J.~A. and Mead, R.
\newblock A simplex method for function minimization.
\newblock \emph{The Computer Journal}, 7\penalty0 (4):\penalty0 308--313, 01
  1965.

\bibitem[Petit et~al.(2025)Petit, Bect, and Vazquez]{petit:2025}
Petit, S.~J., Bect, J., and Vazquez, E.
\newblock Relaxed {G}aussian process interpolation: a goal-oriented approach to
  {B}ayesian optimization.
\newblock \emph{Journal of Machine Learning Research}, 26\penalty0
  (195):\penalty0 1--70, 2025.

\bibitem[Picheny et~al.(2019)Picheny, Vakili, and
  Artemev]{picheny_2019_ordinalbayesianoptimisation}
Picheny, V., Vakili, S., and Artemev, A.
\newblock Ordinal {B}ayesian optimisation, 2019.
\newblock URL \url{https://arxiv.org/abs/1912.02493}.

\bibitem[Picheny et~al.(2022)Picheny, Moss, Torossian, and
  Durrande]{picheny_2022_bayesianquantileexpectileoptimisation}
Picheny, V., Moss, H., Torossian, L., and Durrande, N.
\newblock {B}ayesian quantile and expectile optimisation.
\newblock In Cussens, J. and Zhang, K. (eds.), \emph{Proceedings of the
  Thirty-Eighth Conference on Uncertainty in Artificial Intelligence}, volume
  180 of \emph{Proc. Mach. Learn. Res.}, pp.\  1623--1633. PMLR, 01--05 Aug
  2022.

\bibitem[Pion \& Vazquez(2025)Pion and Vazquez]{pion:2025}
Pion, A. and Vazquez, E.
\newblock Design-marginal calibration of {G}aussian process predictive
  distributions: {B}ayesian and conformal approaches, 2025.
\newblock URL \url{https://arxiv.org/abs/2512.05611}.

\bibitem[Sahoo et~al.(2021)Sahoo, Zhao, Chen, and Ermon]{Sahoo:2021}
Sahoo, R., Zhao, S., Chen, A., and Ermon, S.
\newblock Reliable decisions with threshold calibration.
\newblock In Ranzato, M., Beygelzimer, A., Dauphin, Y., Liang, P., and Vaughan,
  J.~W. (eds.), \emph{Adv. Neural Inf. Process. Syst.}, volume~34, pp.\
  1831--1844. Curran Associates, Inc., 2021.

\bibitem[Schonlau \& Welch(1996)Schonlau and Welch]{schonlau:96:gonff}
Schonlau, M. and Welch, W.~J.
\newblock Global optimization with nonparametric function fitting.
\newblock In \emph{Proceedings of the ASA, Section on Physical and Engineering
  Sciences}, pp.\  183--186. Amer. Statist. Assoc., 1996.

\bibitem[{SciPy Developers}(2026)]{scipy_gennorm}
{SciPy Developers}.
\newblock \emph{{SciPy Project}}, 2026.
\newblock URL
  \url{https://docs.scipy.org/doc/scipy/reference/generated/scipy.stats.gennorm.html}.
\newblock SciPy API Reference, accessed 2026-01.

\bibitem[Scott(1992)]{scott:1992}
Scott, D.~W.
\newblock \emph{Multivariate Density Estimation: Theory, Practice, and
  Visualization}.
\newblock John Wiley \& Sons, New York, 1992.

\bibitem[Shimodaira(2000)]{Shimodaira:2000}
Shimodaira, H.
\newblock Improving predictive inference under covariate shift by weighting the
  log-likelihood function.
\newblock \emph{J. Statist. Plann. Inference}, 90\penalty0 (2):\penalty0
  227--244, 2000.

\bibitem[Srinivas et~al.(2010)Srinivas, Krause, Kakade, and
  Seeger]{srinivas2010:_ucb}
Srinivas, N., Krause, A., Kakade, S.~M., and Seeger, M.
\newblock {G}aussian process optimization in the bandit setting: No regret and
  experimental design.
\newblock In \emph{Proc. 27th International Conference on Machine Learning
  (ICML 2010)}, pp.\  1015--1022, 2010.

\bibitem[Stanton et~al.(2023)Stanton, Maddox, and Wilson]{stanton:2023}
Stanton, S., Maddox, W., and Wilson, A.~G.
\newblock {B}ayesian optimization with conformal prediction sets.
\newblock In Ruiz, F., Dy, J., and van~de Meent, J.~W. (eds.),
  \emph{Proceedings of The 26th International Conference on Artificial
  Intelligence and Statistics}, volume 206 of \emph{Proc. Mach. Learn. Res.},
  pp.\  959--986. PMLR, 25--27 Apr 2023.

\bibitem[Stein(1999)]{stein_interpolation_1999}
Stein, M.~L.
\newblock \emph{Interpolation of Spatial Data: Some Theory for Kriging}.
\newblock Springer Ser. Stat. Springer New York, 1999.

\bibitem[Surjanovic \& Bingham(2013)Surjanovic and
  Bingham]{surjanovic_bingham_optimization}
Surjanovic, S. and Bingham, D.
\newblock Virtual library of simulation experiments: Test functions and
  datasets, 2013.
\newblock URL \url{https://www.sfu.ca/~ssurjano/}.
\newblock Accessed November 2025.

\bibitem[Tom et~al.(2025)Tom, Lo, Corapi, Aspuru-Guzik, and
  Sanchez-Lengeling]{tom:2025}
Tom, G., Lo, S., Corapi, S., Aspuru-Guzik, A., and Sanchez-Lengeling, B.
\newblock Ranking over regression for {B}ayesian optimization and molecule
  selection.
\newblock \emph{APL Machine Learning}, 3\penalty0 (3):\penalty0 036113, 08
  2025.

\bibitem[Tuo \& Wang(2022)Tuo and Wang]{tuo2022uncertainty}
Tuo, R. and Wang, W.
\newblock Uncertainty quantification for {B}ayesian optimization.
\newblock In Camps-Valls, G., Ruiz, F. J.~R., and Valera, I. (eds.),
  \emph{Proceedings of the 25th International Conference on Artificial
  Intelligence and Statistics (AISTATS)}, volume 151 of \emph{Proc. Mach.
  Learn. Res.}, pp.\  2862--2884. PMLR, 2022.

\bibitem[Vazquez(2026)]{gpmp_2026}
Vazquez, E.
\newblock Gpmp: the {G}aussian process micro package, 2026.
\newblock URL \url{https://github.com/gpmp-dev/gpmp}.

\bibitem[Vazquez \& Bect(2010)Vazquez and Bect]{vazquez:2010}
Vazquez, E. and Bect, J.
\newblock Convergence properties of the expected improvement algorithm with
  fixed mean and covariance functions.
\newblock \emph{J. Statist. Plann. Inference}, 140:\penalty0 3088--3095, 11
  2010.

\bibitem[Villemonteix et~al.(2009)Villemonteix, Vazquez, and
  Walter]{villemonteix_informational_2009}
Villemonteix, J., Vazquez, E., and Walter, E.
\newblock An informational approach to the global optimization of
  expensive-to-evaluate functions.
\newblock \emph{J. Global Optim.}, 44\penalty0 (4):\penalty0 509--534, 2009.

\bibitem[Vovk et~al.(2005)Vovk, Gammerman, and Shafer]{vovk_Gammerman_2005}
Vovk, V., Gammerman, A., and Shafer, G.
\newblock \emph{Algorithmic Learning in a Random World}.
\newblock Springer, 2005.

\bibitem[Vovk et~al.(2019)Vovk, Shen, Manokhin, and Xie]{vovk19:_nonpar}
Vovk, V., Shen, J., Manokhin, V., and Xie, M.
\newblock Nonparametric predictive distributions based on conformal prediction.
\newblock \emph{Mach. Learn.}, 108\penalty0 (3):\penalty0 445--474, 2019.

\bibitem[Zhang \& Candès(2024)Zhang and Candès]{zhang_posterior_2024}
Zhang, Y. and Candès, E.~J.
\newblock Posterior conformal prediction, 2024.
\newblock URL \url{https://arxiv.org/abs/2409.19712}.

\end{thebibliography}
\bibliographystyle{icml2026}

\clearpage
\appendix
\onecolumn

\section*{Appendices}

The appendix is organized as follows. Appendix~\ref{apd:gnp} recalls the
generalized normal distribution. Appendix~\ref{sec:proofs-results} gathers the
proofs of the main results. Appendix~\ref{sec:scoring-rules} briefly reviews 
scoring rules. Appendix~\ref{ap:impact_calibration_bo} presents the consistency 
statement discussed in Section~\ref{sec:impact_calibration_bo}. Appendix~\ref
{sec:impl-deta} provides implementation details, and Appendix~\ref
{sec:additional-results} reports additional experimental results.

\section{Generalized Normal Distribution}
\label{apd:gnp}

The generalized normal distribution $\mathcal{GN}(\beta,l,\lambda)$ is a
three-parameter family on $\R$, with shape parameter $\beta>0$, location
parameter $l\in\R$, and scale parameter $\lambda>0$. Its density is
\begin{equation}
g(z)
=
\frac{\beta}{2\,\Gamma(1/\beta)\,\lambda}
\exp\!\left(
-\left(\frac{|z-l|}{\lambda}\right)^{\beta}
\right),
\qquad z\in\R.
\end{equation}
It is symmetric about $l$. The mean is $\E[Z]=l$ and the variance is
\begin{equation}
\Var(Z)=\lambda^2\,\frac{\Gamma(3/\beta)}{\Gamma(1/\beta)}.
\end{equation}
The parameter $\beta$ controls tail decay: $\beta=2$ yields a Gaussian
distribution, $\beta<2$ corresponds to heavier tails, and $\beta>2$ to lighter
tails. In particular,
\begin{equation}
\mathcal{GN}(2,l,\lambda)=\mathcal{N}\left(l,\frac{\lambda^2}{2}\right),
\end{equation}
since $\Gamma(3/2)/\Gamma(1/2)=1/2$.

The cumulative distribution function can be written using the lower incomplete
gamma function
\begin{equation}
\gamma_{\mathrm{inc}}(a,x)=\int_0^x t^{a-1}e^{-t}\,dt.
\end{equation}
For $z\in\R$, the CDF $\Theta_{\beta,l,\lambda}$ of
$\mathcal{GN}(\beta,l,\lambda)$ may be written as
\begin{equation}
\Theta_{\beta,l,\lambda}(z)=\P(Z\le z)=
\begin{cases}
\frac{1}{2}
-\frac{1}{2\,\Gamma(1/\beta)}\,
\gamma_{\mathrm{inc}}\!\left(\frac{1}{\beta},\left(\frac{l-z}{\lambda}\right)^\beta\right),
& z\le l,\\[6pt]
\frac{1}{2}
+\frac{1}{2\,\Gamma(1/\beta)}\,
\gamma_{\mathrm{inc}}\!\left(\frac{1}{\beta},\left(\frac{z-l}{\lambda}\right)^\beta\right),
& z\ge l.
\end{cases}
\end{equation}
We also write $\Theta_\beta:=\Theta_{\beta,0,1}$.

Further properties are given in \citet{nadarajah2005}.
The distribution $\mathcal{GN}(\beta,l,\lambda)$ is available in SciPy as
\texttt{scipy.stats.gennorm}; see \citet{scipy_gennorm}.

\section{Proofs}
\label{sec:proofs-results}

\subsection{Proof of Proposition~\ref{prop:ei_gn}}
\label{ap:ei_gn}

\begin{proof}
Let $\beta>0$ and $\lambda>0$. Write $Z=l+\lambda T$ with
$T\sim\mathcal{GN}(\beta,0,1)$, whose density is
\begin{equation}
g_\beta(t)=\frac{\beta}{2\,\Gamma(1/\beta)}\exp(-|t|^\beta).
\end{equation}
Set $z=a-l$ and $s=z/\lambda$. Then
\begin{equation}
\E\bigl[(a-Z)_+\bigr]
=
\E\bigl[(z-\lambda T)_+\bigr]
=
\lambda\,\E\bigl[(s-T)_+\bigr]
=
\lambda\int_{-\infty}^s (s-t)\,g_\beta(t)\,dt.
\end{equation}
Therefore,
\begin{equation}
\E\bigl[(a-Z)_+\bigr]
=
\lambda\left(
s\,\Theta_\beta(s)-\int_{-\infty}^s t\,g_\beta(t)\,dt
\right).
\end{equation}
It remains to compute $\int_{-\infty}^s t\,g_\beta(t)\,dt$. Since $t\mapsto
t\exp(-|t|^\beta)$ is odd, for any $s\in\R$,
\begin{equation}
\int_{-\infty}^s t\,\exp(-|t|^\beta)\,dt
=
-\int_{|s|}^{\infty} u\,\exp(-u^\beta)\,du.
\end{equation}
By the change of variables $y=u^\beta$,
\begin{equation}
\int_{|s|}^{\infty} u\,\exp(-u^\beta)\,du
=
\frac{1}{\beta}\int_{|s|^\beta}^{\infty} y^{2/\beta-1}\exp(-y)\,dy
=
\frac{1}{\beta}\,\Gamma\left(\frac{2}{\beta},|s|^\beta\right).
\end{equation}
Hence,
\begin{equation}
\int_{-\infty}^s t\,g_\beta(t)\,dt
=
\frac{\beta}{2\,\Gamma(1/\beta)}
\int_{-\infty}^s t\,\exp(-|t|^\beta)\,dt
=
-\frac{1}{2\,\Gamma(1/\beta)}
\Gamma\left(\frac{2}{\beta},|s|^\beta\right).
\end{equation}
Substituting back yields
\begin{equation}
\E\bigl[(a-Z)_+\bigr]
=
\lambda\left(
s\,\Theta_\beta(s)
+\frac{1}{2\,\Gamma(1/\beta)}
\Gamma\left(\frac{2}{\beta},|s|^\beta\right)
\right)
=
z\,\Theta_\beta\left(\frac{z}{\lambda}\right)
+\frac{\lambda}{2\,\Gamma(1/\beta)}
\Gamma\left(\frac{2}{\beta},\left|\frac{z}{\lambda}\right|^\beta\right).
\end{equation}
This is \eqref{eq:ei_gn}.

For $\lambda=0$, $Z=l$ almost surely, so
$\E[(a-Z)_+]=\max(a-l,0)$, which defines $\gamma(z,0,\beta)$.
Continuity at $\lambda=0$ follows by considering three cases. If $z>0$, then
$\Theta_\beta(z/\lambda)\to 1$ and the incomplete-gamma term vanishes. If
$z<0$, then~$\Theta_\beta(z/\lambda)\to 0$ and the same incomplete-gamma term
vanishes. If $z=0$, the first term is identically zero and the second term is
$\lambda \Gamma(2/\beta)/(2\Gamma(1/\beta))$, hence converges to zero.
Strict positivity for $\lambda>0$ holds because $Z$ has a continuous density
that is strictly positive on $\R$, hence $\P(Z<a)>0$ for any finite $a$, and
$(a-Z)_+>0$ on $\{Z<a\}$.
\end{proof}

\subsection{Connection between EI in the Gaussian Case and the Generalized Normal Model}
\label{ap:connection_gaussian_ei}

For $\beta=2$, the generalized normal distribution coincides with a Gaussian:
\begin{equation}
\mathcal{GN}(2,l,\lambda)=\mathcal{N}\left(l,\frac{\lambda^2}{2}\right).
\end{equation}
Let $\sigma=\lambda/\sqrt{2}$ and $\tau=(a-l)/\sigma$. For $\beta=2$ we have
$\Theta_2(s)=\Phi(\sqrt{2}\,s)$, and $\Gamma(1,x)=\exp(-x)$. Substituting
into \eqref{eq:ei_gn} yields
\begin{equation}
\gamma(a-l,\lambda,2)=(a-l)\Phi(\tau)+\sigma\,\varphi(\tau),
\end{equation}
where $\Phi$ and $\varphi$ denote the CDF and PDF of $\mathcal{N}(0,1)$.
This is the usual EI formula for a Gaussian predictive distribution with mean
$l$ and standard deviation $\sigma$.

\subsection{Proof of Proposition~\ref{prop:convergence}}
\label{ap:convergence}

Recall that, at iteration $n$, EI with \tcgp is
\begin{equation}
\rho_n(x)=\gamma\bigl(m_n-f_n(x),\,\lambda_n\,\sigma_n(x),\,\beta_n\bigr),
\qquad
m_n=\min_{1\le i\le n} f(X_i),
\end{equation}
and we set
\begin{equation}
\nu_n(f)=\max_{x\in\X}\rho_n(x)=\rho_n(X_{n+1}),
\qquad
X_{n+1}\in\argmax_{x\in\X}\rho_n(x),
\end{equation}
with $(\beta_n,\lambda_n)\in[\beta_0,\beta_1]\times[\lambda_0,\lambda_1]$ and
$\lambda_0>0$ for all sufficiently large $n$.

\begin{lemma}
\label{lemma:gamma_continuity}
Let $0<\beta_0<\beta_1$. The map
$(z,s,\beta)\mapsto \gamma(z,s,\beta)$, with
$\gamma(z,0,\beta)=z_+$, is continuous on compact subsets
of~$\R\times[0,\infty)\times[\beta_0,\beta_1]$.
\end{lemma}

\begin{proof}
For $s>0$, continuity follows from the closed-form expression
\eqref{eq:ei_gn}. It remains to check continuity at $s=0$, uniformly for
$\beta\in[\beta_0,\beta_1]$. Let $T_\beta\sim\mathcal{GN}(\beta,0,1)$.
Since the map $x\mapsto x_+$ is Lipschitz,
\begin{equation}
\left|\gamma(z,s,\beta)-z_+\right|
=
\left|\E\bigl[(z-sT_\beta)_+\bigr]-z_+\right|
\le
s\,\E|T_\beta|.
\end{equation}
Moreover,
\begin{equation}
\E|T_\beta|=\frac{\Gamma(2/\beta)}{\Gamma(1/\beta)}
\end{equation}
is bounded on $[\beta_0,\beta_1]$ by continuity. Hence
$\gamma(z,s,\beta)\to z_+$ as $s\downarrow0$, uniformly for
$\beta\in[\beta_0,\beta_1]$. If also $z'\to z$,
then~$|z'_+-z_+|\to0$, so continuity holds at $s=0$ as well. This proves the claim.
\end{proof}

\begin{lemma}
\label{lemma:liminf_nu}
For all $f\in\mathcal H$, $\liminf_{n\to\infty}\nu_n(f)=0$.
\end{lemma}

\begin{proof}
Fix $f\in\mathcal H$. Since $\X$ is compact, $(X_n)$ has a cluster point
$\tilde x\in\X$ and there exists a strictly increasing map $\phi:\N\to\N$ such
that $X_{\phi(n)}\to \tilde x$. Since $k$ is continuous and $\X$ is compact,
every $f\in\mathcal H$ is continuous on $\X$, hence $f(X_{\phi(n)})\to f(\tilde x)$.

By \citet[Prop.~10]{vazquez:2010},
\begin{equation}
f_{\phi(n)-1}(X_{\phi(n)})\to f(\tilde x),
\qquad
\sigma_{\phi(n)-1}(X_{\phi(n)})\to 0.
\end{equation}
Moreover, since $\phi$ is increasing,
\begin{equation}
m_{\phi(n)-1}\le m_{\phi(n-1)}\le f(X_{\phi(n-1)})\to f(\tilde x),
\end{equation}
hence $\limsup_{n\to\infty} m_{\phi(n)-1}\le f(\tilde x)$. Set
\begin{equation}
z_{\phi(n)-1}:=m_{\phi(n)-1}-f_{\phi(n)-1}(X_{\phi(n)}).
\end{equation}
Then $\limsup_{n\to\infty} z_{\phi(n)-1}\le 0$. Define also
\begin{equation}
t_{\phi(n)-1}:=\lambda_{\phi(n)-1}\,\sigma_{\phi(n)-1}(X_{\phi(n)}),
\end{equation}
so that $t_{\phi(n)-1}\to 0$ since $0<\lambda_{\phi(n)-1}\le \lambda_1$.

Fix $\varepsilon>0$. For $n$ large enough, $z_{\phi(n)-1}\le \varepsilon$.
For $\lambda>0$, $z\mapsto \gamma(z,\lambda,\beta)=\E[(z-Z)_+]$ is nondecreasing,
hence
\begin{align}
\nu_{\phi(n)-1}(f)
&=\rho_{\phi(n)-1}(X_{\phi(n)}) \\
&=\gamma\bigl(z_{\phi(n)-1},\,t_{\phi(n)-1},\,\beta_{\phi(n)-1}\bigr) \\
&\le \gamma\bigl(\varepsilon,\,t_{\phi(n)-1},\,\beta_{\phi(n)-1}\bigr).
\end{align}
Since $t_{\phi(n)-1}\to 0$ and, for all large $n$,
$\beta_{\phi(n)-1}\in[\beta_0,\beta_1]$, every subsequence of
$\gamma(\varepsilon,t_{\phi(n)-1},\beta_{\phi(n)-1})$ has a further subsequence
along which $\beta_{\phi(n)-1}\to \bar\beta\in[\beta_0,\beta_1]$.
Along this further subsequence, Lemma~\ref{lemma:gamma_continuity} gives
\begin{equation}
\gamma(\varepsilon,t_{\phi(n)-1},\beta_{\phi(n)-1})
\longrightarrow
\gamma(\varepsilon,0,\bar\beta)
=
\varepsilon.
\end{equation}
Thus the full sequence converges to $\varepsilon$.
Therefore,
\begin{equation}
\limsup_{n\to\infty}\nu_{\phi(n)-1}(f)\le \varepsilon.
\end{equation}
Since $\varepsilon>0$ is arbitrary, $\liminf_{n\to\infty}\nu_n(f)=0$.
\end{proof}

\begin{proof}[Proof of Proposition~\ref{prop:convergence}]
Assume by contradiction that $(X_n)$ is not dense in $\X$. Then there exists
$\tilde x\in\X$ such that $\tilde x$ is not adherent to $\{X_n:n\ge 1\}$.
By the NEB property, $\sigma_n^2(\tilde x)$ does not converge to $0$.
Since $(\sigma_n^2(\tilde x))_{n\ge 1}$ is nonincreasing, it converges to some
$c^2>0$, hence $\inf_{n\ge 1}\sigma_n(\tilde x)\ge c$.

Since $f_n$ is the $\mathcal H$-projection of $f$ onto the span of kernel
sections, $\|f_n\|_{\mathcal H}\le \|f\|_{\mathcal H}$, hence
\begin{equation}
|f_n(\tilde x)|
\le \|f_n\|_{\mathcal H}\sqrt{k(\tilde x,\tilde x)}
\le \|f\|_{\mathcal H}\sqrt{k(\tilde x,\tilde x)}.
\end{equation}
Also, $(m_n)$ is bounded because $f$ is continuous on the compact set $\X$.
Therefore, $z_n^*:=m_n-f_n(\tilde x)$ ranges in a compact interval $I\subset\R$ and
$\sigma_n(\tilde x)$ ranges in $[c,\sqrt{k(\tilde x,\tilde x)}]$.

Define
\begin{equation}
\psi(z,\sigma,\beta,\lambda)=\gamma(z,\lambda\sigma,\beta).
\end{equation}
The map $\psi$ is continuous by Lemma~\ref{lemma:gamma_continuity}. Consider the compact set
\begin{equation}
K
=
I\times[c,\sqrt{k(\tilde x,\tilde x)}]\times[\beta_0,\beta_1]\times[\lambda_0,\lambda_1].
\end{equation}
Since $\lambda_0>0$ and $\gamma(\cdot,\lambda,\beta)$ is strictly positive for
$\lambda>0$, we have $\psi>0$ on $K$, hence
\begin{equation}
\eta:=\inf_{(z,\sigma,\beta,\lambda)\in K}\psi(z,\sigma,\beta,\lambda)>0.
\end{equation}
For all sufficiently large $n$,
\begin{equation}
\rho_n(\tilde x)
=\gamma(z_n^*,\lambda_n\sigma_n(\tilde x),\beta_n)
\ge \eta.
\end{equation}
Since $X_{n+1}$ maximizes $\rho_n$, we obtain
\begin{equation}
\nu_n(f)=\rho_n(X_{n+1})\ge\rho_n(\tilde x)\ge\eta
\qquad\text{for all sufficiently large $n$,}
\end{equation}
which contradicts Lemma~\ref{lemma:liminf_nu}. Therefore, $(X_n)$ is dense in $\X$.
\end{proof}

\section{Scoring Rules}
\label{sec:scoring-rules}

\subsection{General Definition}
\label{ap:scoring_rules}

To assess calibration and sharpness jointly, we use \emph{proper scoring rules}
\citep{gneiting07:scoring_rules}. Scoring rules assign a numerical score
$S(\hat F,z)$ to a predictive CDF $\hat F$ and an outcome $z\in\R$. A scoring rule is
\emph{strictly proper} if, for the true distribution $F$,
\begin{equation}
  \E_{Z\sim F}[S(F,Z)] \;\le\; \E_{Z\sim F}[S(\hat F,Z)],
\end{equation}
with equality if and only if $\hat F=F$.

For goal-oriented prediction, we use the \emph{threshold-weighted continuous
ranked probability score} (twCRPS) \citep{allen:2023}, defined for a weight
function $w$, for instance $w(u)=\one\{u\le t\}$, by
\begin{equation}
    S_{\rm twCRPS}(\hat F, z)
    = \int (\hat F(u) - \one\{u \ge z\})^2 w(u)\,\mathrm{d}u .
\end{equation}
For nonconstant weights, twCRPS is a proper weighted score and does not identify
the full predictive distribution outside the weighted region.

Given a true CDF $F$, a predictive CDF $\hat F$, and a scoring rule $S$, we
evaluate predictive performance using the expected score
\begin{equation}
    J_S(\hat F) = \E_{Z\sim F} (S(\hat F, Z)).
\end{equation}

An empirical version is given below. In the spatial setting, the expected score
of the family of predictive CDFs $\hat F_n(\cdot\mid x)$, indexed by $x\in\X$,
is
\begin{equation}
    J_{S, \mu}(\hat F_n) = \E_n\left[S(\hat F_n(\cdot \mid X), f(X))\right].
\end{equation}

\subsection{Empirical twCRPS}
\label{app:twscrps}

Suppose that $w:\R\rightarrow\R^+$ is a weight function and $\hat F$ a
predictive CDF. A \emph{chaining} function is any function $v$
satisfying~$v(x') - v(x) = \int_{x}^{x'} w(u)\mathrm{d}u$ \citep[see, e.g.,][]
{allen:2023}. The twCRPS at $z\in \R$ can be rewritten as
\begin{equation}
    S_{\rm twCRPS}(\hat F, z)
    =
    \E_{\hat F}\bigl[|v(Y) - v(z)|\bigr]
    - \frac{1}{2}\E_{\hat F}\bigl[|v(Y) - v(Y')|\bigr],
\end{equation}
where $Y$ and $Y'$ are independent random variables with CDF $\hat F$.

With $w(z)=\one\{z\le t\}$ for a threshold $t\in\R$, we may take
$v(z)=\min(z,t)$.

For the empirical forecast-observation
pair $(Z, \hat F)$ and a sample $(Y_\ell)_{\ell=1}^K$ from $\hat F$, the twCRPS
at $Z$ can be approximated by the Monte Carlo estimator
\begin{equation}
    \widehat S_{\rm twCRPS,K}(\hat F, Z)
    =
    \frac{1}{K}\sum_{\ell=1}^K |v(Y_\ell)-v(Z)|
    -
    \frac{1}{2K^2}\sum_{\ell=1}^K\sum_{r=1}^K |v(Y_\ell)-v(Y_r)|.
\end{equation}
For multiple forecast-observation pairs $(Z_i, \hat F_i)_{i=1}^m$, the
empirical score is
\begin{equation}
    \widehat J_{\rm twCRPS,m}
    =
    \frac{1}{m}\sum_{i=1}^m
    \widehat S_{\rm twCRPS,K}(\hat F_i,Z_i).
\end{equation}

\section{Impact of Calibration at the Current Best Value}
\label{ap:impact_calibration_bo}

This appendix provides the consistency statement mentioned in
Section~\ref{sec:impact_calibration_bo}. It shows that, if the set of global
minimizers has~$\mu$-measure zero, enforcing occurrence calibration at the
current best value $m_n$ forces the predicted probability mass below~$m_n$ to
collapse as $m_n$ approaches the global minimum.

\begin{proposition}[Vanishing tail mass with occurrence calibration at $m_n$]
\label{prop:tailmass_vanish}
Assume that $m_n \overset{\rm a.s.}\longrightarrow f(x^*)$ for some global
minimizer $x^*\in\arg\min_{x\in\X} f(x)$. If $\hat F_n$ is occurrence-calibrated
at threshold $m_n$, that is,
\begin{equation}
\label{eq:mu_calibration_at_mn}
\E_n\bigl[\hat F_n(m_n\mid X)\bigr] = \P_n\bigl(f(X)\le m_n\bigr),
\qquad X\sim\mu,
\end{equation}
then for every $\varepsilon\in(0,1]$, we have
\begin{equation}
\label{eq:tailmass_vanish_bound}
\limsup_{n\to\infty}\P_n\bigl(\hat F_n(m_n\mid X)>\varepsilon\bigr)
\le \varepsilon^{-1}\, \mu\left(\{x\in\X: f(x)=f(x^*)\}\right)
\qquad \text{a.s.}
\end{equation}
In particular, if $\mu\left(\{x\in\X: f(x)=f(x^*)\}\right)=0$, then
\begin{equation}
\P_n\bigl(\hat F_n(m_n\mid X)>\varepsilon\bigr)\to 0
\qquad \text{a.s.}
\end{equation}
\end{proposition}

\begin{proof}
Fix $\varepsilon\in(0,1]$ and set $A_n:=\{\hat F_n(m_n\mid X)>\varepsilon\}$.
By Markov's inequality,
$$
\P_n(A_n)
\le
\varepsilon^{-1}\,\E_n\bigl[\hat F_n(m_n\mid X)\bigr].
$$
Using \eqref{eq:mu_calibration_at_mn}, this yields
\begin{equation}
\label{eq:reduce_to_excursion_prob}
\P_n\bigl(\hat F_n(m_n\mid X)>\varepsilon\bigr)
\le
\varepsilon^{-1}\,\P_n\bigl(f(X)\le m_n\bigr).
\end{equation}

On any sample path such that $m_n\to f(x^*)$, we have
$f(x)\ge f(x^*)$ for all $x\in\X$, and therefore
\begin{equation}
\label{eq:indicator_limit}
\mathbf{1}\{f(x)\le m_n\} \to \mathbf{1}\{f(x)=f(x^*)\}
\qquad \text{for all } x\in\X .
\end{equation}
Moreover, $0\le \mathbf{1}\{f(x)\le m_n\}\le 1$, so dominated convergence with
respect to $\mu$ gives
\begin{equation}
\label{eq:excursion_prob_limit}
\P_n\bigl(f(X)\le m_n\bigr)
=
\E_n\bigl[\mathbf{1}\{f(X)\le m_n\}\bigr]
\to
\E_n\bigl[\mathbf{1}\{f(X)=f(x^*)\}\bigr]
=
\mu\left(\{x\in\X: f(x)=f(x^*)\}\right)
\qquad \text{a.s.}
\end{equation}
Combining \eqref{eq:reduce_to_excursion_prob} with \eqref{eq:excursion_prob_limit}
and taking $\limsup$ yields \eqref{eq:tailmass_vanish_bound}. The last statement
follows when 
\begin{equation}
    \mu\left(\{x\in\X: f(x)=f(x^*)\}\right)=0.
\end{equation}
\end{proof}

\paragraph{Interpretation.}
Proposition~\ref{prop:tailmass_vanish} is a population statement with an
auxiliary draw $X\sim\mu$ (independent of $\Dcal_n$). It implies that, when the
set of global minimizers has $\mu$-measure $0$, enforcing occurrence
calibration at $m_n$ forces $\hat F_n(m_n\mid X)\to 0$
in~$\mu$-probability. For every $\varepsilon\in(0,1]$, the $\mu$-mass of the
inputs at which $\hat F_n(m_n\mid x)$ exceeds $\varepsilon$ therefore vanishes
as $n\to\infty$. Equivalently, by~\eqref{eq:mu_calibration_at_mn},
the $\mu$-average predicted mass $\E_n[\hat F_n(m_n\mid X)]$ tends to zero.

The result does not describe how the predictive lower tail behaves in the
region explored by the BO policy. This motivates calibrating below a higher
threshold $t_n$ and using notions based on the truncated CDF
$\hat F_{n,t_n}(\cdot\mid x)$.

\section{Implementation Details}
\label{sec:impl-deta}

\subsection{GP Matérn Model and Parameters}
\label{ap:parameters}

In the experiments of Section~\ref{sec:experiment_calibration}, the model
$\GP(m,k)$ has constant mean $m(x)=m_0$ for $x\in\XX$. The kernel $k$ belongs to
the anisotropic Matérn family, defined for $x,y\in\mathbb{R}^d$ by
\begin{equation}
k_{\sigma,\nu,\rho}(x,y)
= \sigma^2 \kappa_\nu(h), \qquad
h^2 = \sum_{i=1}^d \frac{(x_{[i]}-y_{[i]})^2}{\rho_i^2},
\end{equation}
where $\sigma^2$ is the variance parameter, the components $\rho_i$ of
$\rho=(\rho_1,\ldots,\rho_d)$ are the lengthscales, and $\kappa_\nu$ denotes
the Matérn correlation function \citep[Chapter~2.7]{stein_interpolation_1999}.
The smoothness parameter is fixed to $\nu=p+1/2$, with
$p\in\mathbb{N}^\star$.

The parameters $(m_0,\rho,\sigma)$ are selected by maximum likelihood. We use
$p=2$ in the experiments of Section~\ref{sec:experiment_calibration}.

\subsection{Maximization of EI}
\label{ap:max_ei}
EI is maximized with the sequential Monte Carlo (SMC) procedure of
\citet{Feliot_2016}, available in \texttt{gpmp}. A population of particles is
iteratively reweighted with EI-proportional weights and propagated through
mutation steps, yielding a
discrete approximation of the EI maximizers. The particle with the largest EI
is then used to initialize a local maximization with the \emph{sequential least
squares programming} (SLSQP) algorithm. This hybrid global--local
procedure is applied for all methods. We use $1000$ particles in the experiments
of Section~\ref{sec:experiment_calibration}.

\subsection{Bayesian Optimization with \tcgp}
\label{ap:optimization_algo}

Algorithm~\ref{alg:calibrated_bo} summarizes the EI-based BO procedure with
\tcgp\ described in the main text. At iteration $n$, we consider the empirical
quantile $q_{\delta,n}$ as a candidate threshold and estimate the excursion
probability below it. If $\hat p^{\,w}_{q_{\delta,n},n}\ge p_{\min}$, we
set~$t_n=q_{\delta,n}$. Otherwise, we keep the previous threshold. This condition
is only a safeguard against calibrating on too few points as the adaptive design
evolves.

\begin{algorithm}[h!]
\caption{BO with \tcgp}
\label{alg:calibrated_bo}
\begin{algorithmic}[1]
\REQUIRE Initial dataset $\Dcal_{n_0}=\{(X_i,Z_i)\}_{i=1}^{n_0}$ with $Z_i=f(X_i)$,
design space $\X$, budget $N_{\max}$, quantile level $\delta\in(0,1]$, minimum excursion
probability $p_{\min}>0$.
\ENSURE Best value $m_{N_{\max}}$ and associated point $X_{\min}$.

\STATE Set $n\leftarrow n_0$.
\STATE Set $m_n \leftarrow \min_{1\le i\le n} Z_i$ and choose any
$X_{\min}\in\arg\min_{1\le i\le n} Z_i$.
\STATE Set the current threshold $t \leftarrow q_{\delta,n}$.

\WHILE{$n < N_{\max}$}
  \STATE Fit a GP model to $\Dcal_n$ by maximum likelihood and obtain functions
  $x\mapsto f_n(x)$ and $x\mapsto \sigma_n(x)$.
  \STATE Estimate the importance weights $(\tilde w_i)_{i=1}^n$ (Section~\ref{sec:metrics_below_threshold}
  and paragraph below).
  \STATE Set $\tilde t\leftarrow q_{\delta,n}$.
  \STATE Compute the weighted excursion frequency
  $$\hat p^{\,w}_{\tilde t,n}=\sum_{i=1}^n \tilde w_i\,\mathbf{1}\{Z_i\le \tilde t\}.$$
  \IF{$\hat p^{\,w}_{\tilde t,n}\ge p_{\min}$}
    \STATE Update $t \leftarrow \tilde t$.
  \ENDIF
  \STATE Select $(\beta_n,\lambda_n)$ by minimizing $J^{\mathrm{LOO},w}_{t,n}$ on
  $[\beta_0,\beta_1]\times[\lambda_0,\lambda_1]$ (Section~\ref{sec:selection_parameters}
  and paragraph below).
  \STATE Define EI $\rho_n$ using the \tcgp predictive CDFs with parameters
  $(\beta_n,\lambda_n)$ (Section~\ref{sec:opt_calibrated}).
  \STATE Choose $X_{n+1}\in\arg\max_{x\in\X}\rho_n(x)$ (paragraph below).
  \STATE Evaluate $Z_{n+1}=f(X_{n+1})$ and set $\Dcal_{n+1}\leftarrow \Dcal_n\cup\{(X_{n+1},Z_{n+1})\}$.
  \IF{$Z_{n+1} < m_n$}
    \STATE Set $m_{n+1}\leftarrow Z_{n+1}$ and $X_{\min}\leftarrow X_{n+1}$.
  \ELSE
    \STATE Set $m_{n+1}\leftarrow m_n$.
  \ENDIF
  \STATE Set $n\leftarrow n+1$.
\ENDWHILE

\STATE \textbf{Return} $(X_{\min},m_{N_{\max}})$.
\end{algorithmic}
\end{algorithm}

\paragraph{Density-ratio estimation and weights.}
The weighted LOO quantities in Section~\ref{sec:metrics_below_threshold} aim at
approximating $\mu$--averages using the adaptively sampled BO locations
$(X_i)_{i=1}^n$, which are not distributed according to $\mu$. We use an
empirical density-ratio correction as a heuristic for this adaptive design. If
$\nu_n$ denotes the smoothed empirical design density of the BO locations, the
idealized weights have the form
\begin{equation}
w_i \propto \frac{\mathrm{d}\mu}{\mathrm{d}\nu_n}(X_i),
\end{equation}
and are used to approximate, for the integrands $\varphi$ appearing in the
weighted LOO criteria,
\begin{equation}
\int_{\X}\varphi(x)\,\mu(\mathrm{d}x)
\approx
\sum_{i=1}^n \tilde w_i\,\varphi(X_i),
\qquad
\tilde w_i=\frac{w_i}{\sum_{j=1}^n w_j}.
\end{equation}
In our experiments, $\mu$ is uniform on $\X$, hence
\begin{equation}
w_i \propto \frac{1}{\nu_n(X_i)}.
\end{equation}

In practice, we use a KDE estimate $\hat\nu_n$ of this design density.
We rescale the design points $X_i$ to $[0,1]^d$ using the
bounds of $\X$, and fit a Gaussian kernel density estimator (KDE) to the
rescaled points \citep{scott:1992}. The bandwidth is chosen according to Scott's
rule \citep{scott:1992}. We then set
\begin{equation}
w_i \propto \frac{1}{\hat\nu_n(X_i)},
\end{equation}
and normalize to obtain $(\tilde w_i)_{i=1}^n$.

Large density ratios can lead to highly variable weights. A standard safeguard
is weight clipping, for instance $w_i\leftarrow~\min(w_i,w_{\max})$ for a fixed
$w_{\max}>0$, followed by renormalization. We did not use clipping in the
reported experiments.

\paragraph{Selection of $(\beta_n,\lambda_n)$.}
We minimize $J^{\mathrm{LOO},w}_{t_n,n}$ with a two-stage procedure. First, we
sample $(\beta,\lambda)$ uniformly on $[\beta_0,\beta_1]\times~[\lambda_0,
\lambda_1]$
($900$ candidates) and select the best pair $(\tilde\beta,\tilde\lambda)$ according to
$J^{\mathrm{LOO},w}_{t_n,n}$. Second, we refine with the Nelder--Mead  
algorithm \citep{nelder_mead} 
initialized at $(\tilde\beta,\tilde\lambda)$. The supremum over $u\in[0,1]$ in
$J^{\mathrm{LOO},w}_{t_n,n}$ is evaluated on a fixed grid of $u$ values.

\subsection{Test-Set Evaluation of $\mu$--Calibration Metrics}
\label{ap:metrics_testset}

When an independent evaluation set $(X'_j,Z'_j)_{j=1}^L$ is available, with the
$X'_j$ drawn i.i.d.\ from $\mu$ and $Z'_j=f(X'_j)$, $\mu$--calibration metrics
can be computed without leave-one-out or reweighting. The predictive CDFs
$\hat F_n(\cdot\mid x)$ are constructed from $\Dcal_n$ only.

\paragraph{Direct test-set estimators.}
Fix $t\in\R$. Define
\begin{equation}
U'_{t,j}
=
\begin{cases}
\hat F_n(Z'_j\mid X'_j)\big/\hat F_n(t\mid X'_j), & Z'_j\le t,\\
1, & Z'_j>t,
\end{cases}
\end{equation}
with the convention that $\hat F_n(t\mid X'_j)>0$ when $Z'_j\le t$. Assume that
$\sum_{j=1}^L \mathbf{1}\{Z'_j\le t\}>0$ and set
\begin{equation}
G'_{t,L}(u)
=
\frac{
\sum_{j=1}^L \mathbf{1}\{Z'_j\le t\}\,\mathbf{1}\{U'_{t,j}\le u\}
}{
\sum_{j=1}^L \mathbf{1}\{Z'_j\le t\}
},
\qquad
J'_{\mathrm{tKS\text{-}PIT},L}(\hat F_n\mid t)
=
\sup_{u\in[0,1]}\left|G'_{t,L}(u)-u\right|.
\end{equation}
The excursion probability $p_t=\P(f(X)\le t)$ is estimated by
\begin{equation}
\hat p_{t,L}
=
\frac{1}{L}\sum_{j=1}^L \mathbf{1}\{Z'_j\le t\},
\end{equation}
and the corresponding occurrence discrepancy estimator is
\begin{equation}
r'_{t,L}
=
\biggl|
\hat p_{t,L}
-
\frac{1}{L}\sum_{j=1}^L \hat F_n(t\mid X'_j)
\biggr|.
\end{equation}

\paragraph{SMC-based estimators for small thresholds.}
For small $t$ (in particular $t=m_n$ at later BO iterations), direct sampling
from $\mu$ yields too few points with $f(X)\le t$ to estimate $G_{\mu,t}$ and
$p_t$ accurately. We therefore use subset simulation implemented via sequential
Monte Carlo (SMC) \citep{bect2017bayesian} to obtain particles approximately 
distributed according to $\mu(\cdot\mid f(X)\le t)$
and an estimate of $p_t=\P(f(X)\le t)$.

At the initial iteration with $n_0$ observations, we run subset simulation down 
to $t=m_{n_0}$ to compute an estimate $\hat p_{m_{n_0}}$ and a particle cloud
$(\tilde X_i^{(0)})_{i=1}^m\approx \mu(\cdot\mid f(X)\le m_{n_0})$. We evaluate
tKS--PIT at $m_{n_0}$ using $(\tilde X_i^{(0)})_{i=1}^m$, and we evaluate the
occurrence discrepancy using $\hat p_{m_{n_0}}$.

When BO produces a new best value $m_n<m_{n-1}$, we do not rerun subset
simulation from scratch. We reuse the particles available at level $m_{n-1}$ as
an initial population and perform additional SMC steps targeting the new
excursion set~$\{x:\, f(x)\le m_n\}$. This yields updated particles
$(\tilde X_i^{(n)})_{i=1}^m\approx \mu(\cdot\mid f(X)\le m_n)$ together with an
updated estimate $\hat p_{m_n}$. The particle cloud is used to form the
empirical distribution of the $\mu$-tPIT below $m_n$ (and hence tKS--PIT), while
$\hat p_{m_n}$ is used in the occurrence discrepancy.

In all BO experiments, we use $m=1000$ particles and Metropolis--Hastings moves 
for the Markov transition steps.

\subsection{Reweighting for \bcrgp}
\label{ap:bcrgp_reweight}

The \bcrgp\ approach of \citet{pion:2025} specifies a generalized normal model
for standardized residuals and places a uniform prior on
$\theta=(\beta,\lambda)$. We recall the resulting posterior construction and
describe how we incorporate importance weights to target $\mu$-probabilistic
calibration over $\X$. In our experiments, global calibration diagnostics for
\bcrgp\ are reported using KS--PIT, i.e., the limit of tKS--PIT as $t\to\infty$.

Let $R_{n,-i}(X_i)$ denote the leave-one-out standardized residual at $X_i$,
\begin{equation}
R_{n,-i}(X_i)=\frac{Z_i-f_{n,-i}(X_i)}{\sigma_{n,-i}(X_i)},
\end{equation}
where $f_{n,-i}$ and $\sigma_{n,-i}$ are the GP predictive mean and standard
deviation obtained from $\Dcal_n\setminus\{(X_i,Z_i)\}$. Let
$p(r\mid\theta)$ be the density of $\mathcal{GN}(\beta,0,\lambda)$ evaluated at
$r$.

With a uniform prior on $(0,a)\times(0,b)$, \bcrgp\ defines
\begin{equation}
p \left(\theta \,\middle|\, \{R_{n,-i}(X_i)\}_{i=1}^n\right)
\propto
\prod_{i=1}^n p\!\left(R_{n,-i}(X_i)\mid\theta\right)\,
\mathbf{1}\{0<\beta<a\}\,\mathbf{1}\{0<\lambda<b\}.
\end{equation}

Let $\hat\nu_n$ be the density estimate used above. To incorporate importance
weights, we replace the standard likelihood by a weighted likelihood, following
\citet{Shimodaira:2000}. Let $(\omega_i)_{i=1}^n$ be nonnegative weights proportional
to $1/\hat\nu_n(X_i)$
(Section~\ref{sec:metrics_below_threshold}). For the weighted likelihood, the
global scale of the exponents affects posterior concentration, so we fix the
normalization
\begin{equation}
\bar w_i = n\,\omega_i\bigg/\sum_{j=1}^n \omega_j,
\qquad
\sum_{i=1}^n \bar w_i=n.
\end{equation}
The resulting pseudo-posterior is
\begin{equation}
p_w \left(\theta \,\middle|\, \{R_{n,-i}(X_i)\}_{i=1}^n\right)
\propto
\prod_{i=1}^n p\!\left(R_{n,-i}(X_i)\mid\theta\right)^{\bar w_i}\,
\mathbf{1}\{0<\beta<a\}\,\mathbf{1}\{0<\lambda<b\}.
\end{equation}

Given Monte Carlo samples $(\beta_i,\lambda_i)_{i=1}^m$, \citet{pion:2025}
derive a Bayesian selection rule for $(\beta^*,\lambda^*)$ targeting
probabilistic calibration, quantified by the KS--PIT up to a prescribed
quantile. To limit computational cost, we set $m=200$ in this work. Increasing
$m$ further had negligible impact on BO performance.

\subsection{Experimental Parameter Settings}
\label{ap:parameters_settings}

We summarize below the main parameters used in all numerical experiments.

\paragraph{Test-set evaluation parameters}
\begin{itemize}
\item Test-set evaluation of calibration metrics: The subset simulation
procedure uses $k_0 = 1000$ particles to estimate~$\P(f(X)\le m_n)$ and to
generate samples from $\mu(\cdot\mid f(X)\le m_n)$.

\item Occurrence discrepancy: The occurrence discrepancy $r_{m_n,n}$ is computed
using an additional $k_1 = 1000$ points sampled uniformly over~$\XX$.

\item tKS--PIT computation: To compute the tKS--PIT, we subsample
  $k_2 = 900$ particles from the subset simulation output, ensuring
  that all corresponding function values lie below the current
  threshold $m_n$ at each iteration.
\end{itemize}

\paragraph{Model parameters}

\begin{itemize}
    \item GP prior regularity: We use a Matérn covariance kernel with 
    smoothness parameter $\nu = p + \tfrac{1}{2}$, with $p = 2$, corresponding 
    to twice mean-square differentiable sample paths. All GP hyperparameters 
    are estimated by maximum likelihood.
    \item Parameter bounds for calibration: The calibration parameters are 
    constrained to
\[
\lambda \in [\lambda_0,\lambda_1] = [5\times 10^{-3},\, 10],
\qquad
\beta \in [\beta_0,\beta_1] = [0.1,\, 10].
\]
    \item $p_{\min} = 0.015$.
\end{itemize}

\subsection{Test Functions Used in the Experiments}
\label{ap:test_functions}

\begin{table*}[h!]
  \centering
  \setlength\tabcolsep{2pt}
  \footnotesize
\begin{tabularx}{\textwidth}{l l Y}
\toprule
\textbf{Name} & \textbf{Domain} & \textbf{Expression of $f(x)$}\\
\midrule
Goldstein--Price & $[-2,2]^2$ &
$(1 + (x_1 + x_2 + 1)^2(19 - 14x_1 + 3x_1^2 - 14x_2 + 6x_1x_2 + 3x_2^2))(30 + (2x_1 - 3x_2)^2(18 - 32x_1 + 12x_1^2 + 48x_2 - 36x_1x_2 + 27x_2^2))$ \\
Rosenbrock & $[-5,10]^d$ &
$\sum_{i=1}^{d-1} \bigl[100(x_{i+1}-x_i^2)^2 + (x_i-1)^2\bigr]$ \\
Ackley & $[-32.768,32.768]^d$ &
$-20\exp\left(-0.2\sqrt{\frac{1}{d}\sum_{i=1}^d x_i^2}\right) - \exp\left(\frac{1}{d}\sum_{i=1}^d \cos(2\pi x_i)\right) + 20 + e$ \\
Dixon--Price--$d$ & $[-10,10]^d$ &
$(x_1 - 1)^2 + \sum_{i=2}^{d} i(2x_i^2 - x_{i-1})^2$ \\
Hartmann--$d$ & $[0,1]^d$ &
$-\sum_{i=1}^{4} c_i^{(d)} \exp\bigl(-\sum_{j=1}^{d} a_{ij}^{(d)}(x_j - p_{ij}^{(d)})^2\bigr)$ \\
Michalewicz--$d$ & $[0,\pi]^d$ & $-\sum_{i=1}^{d} \sin(x_i)\sin^{2m}(i x_i^2/\pi)$ \\
Cross-In-Tray & $[-10, 10]^2$ & $-0.0001\left(\left|\sin(x_1)\sin(x_2)\exp\left(\left|100 - \frac{\sqrt{x_1^2 + x_2^2}}{\pi}\right|\right)\right|+ 1 \right)^{0.1}$ \\
Shekel--$m$ & $[0, 10]^4$ & $-\sum_{i=1}^{m} \left(\sum_{j=1}^4 (x_j - C_{ji})^2 + \beta_i \right)^{-1}$ \\
Perm--$d$ & $[-d, d]^d$ & $\sum_{i=1}^d \left(\sum_{j=1}^d (j+\beta_{\mathrm P})\left(x_j^i - \frac{1}{j^i}\right)\right)^2$ \\
\bottomrule
\end{tabularx}

\caption{Test functions used in the numerical experiments.
  Each function is continuous, deterministic, and defined on the domain~$\XX$.
  Expressions and domains follow the standard definitions provided in
  \citet{surjanovic_bingham_optimization}, with parameter values specified
  below.}
\label{tab:test_functions}
\end{table*}

\medskip
\noindent\textbf{Hartmann parameters.}
\scriptsize
$$
a^{(3)}=\begin{bmatrix}
3.0 & 10.0 & 30.0\\
0.1 & 10.0 & 35.0\\
3.0 & 10.0 & 30.0\\
0.1 & 10.0 & 35.0
\end{bmatrix},\quad
c^{(3)}=\begin{bmatrix}1.0 \\ 1.2 \\ 3.0 \\ 3.2\end{bmatrix},
\quad
p^{(3)}=10^{-4}\begin{bmatrix}
3689 & 1170 & 2673\\
4699 & 4387 & 7470\\
1091 & 8732 & 5547\\
381 & 5743 & 8828
\end{bmatrix}.
$$
$$
a^{(6)}=\begin{bmatrix}
10 & 3 & 17 & 3.5 & 1.7 & 8\\
0.05 & 10 & 17 & 0.1 & 8 & 14\\
3 & 3.5 & 1.7 & 10 & 17 & 8\\
17 & 8 & 0.05 & 10 & 0.1 & 14
\end{bmatrix},\quad
c^{(6)}=\begin{bmatrix}1.0 \\ 1.2 \\ 3.0 \\ 3.2\end{bmatrix},\quad
p^{(6)}=10^{-4}\begin{bmatrix}
1312 & 1696 & 5569 & 124 & 8283 & 5886\\
2329 & 4135 & 8307 & 3736 & 1004 & 9991\\
2348 & 1451 & 3522 & 2883 & 3047 & 6650\\
4047 & 8828 & 8732 & 5743 & 1091 & 381
\end{bmatrix}.
$$

\normalsize
\medskip
\noindent\textbf{Michalewicz parameter.}
\scriptsize
$$m=10$$

\normalsize
\medskip
\noindent\textbf{Perm parameter.}
\scriptsize
$$\beta_{\mathrm P}=1$$

\normalsize
\medskip
\noindent\textbf{Shekel parameters.}
\scriptsize
$$
\beta = \frac{1}{10}(1,2,2,4,4,6,3,7,5,5)^T
$$

$$
C=\begin{bmatrix}
4 & 1 & 8 & 6 & 3 & 2 & 5 & 8 & 6 & 7 \\
4 & 1 & 8 & 6 & 7 & 9 & 3 & 1 & 2 & 3.6 \\
4 & 1 & 8 & 6 & 3 & 2 & 5 & 8 & 6 & 7 \\
4 & 1 & 8 & 6 & 7 & 9 & 3 & 1 & 2 & 3.6
\end{bmatrix}.
$$
\normalsize

\section{Additional Experiments}
\label{sec:additional-results}

\subsection{Running Time of the Methods}
\label{sec:computation-time}

We report the computing time per iteration for each predictive model, excluding
maximization of the sampling criterion. The times are similar for GP, \tcgp,
\onego, and \bcrgp. \regp is about ten times slower.

\begin{figure}[htbp]
    \centering
        \includegraphics[width=\textwidth]{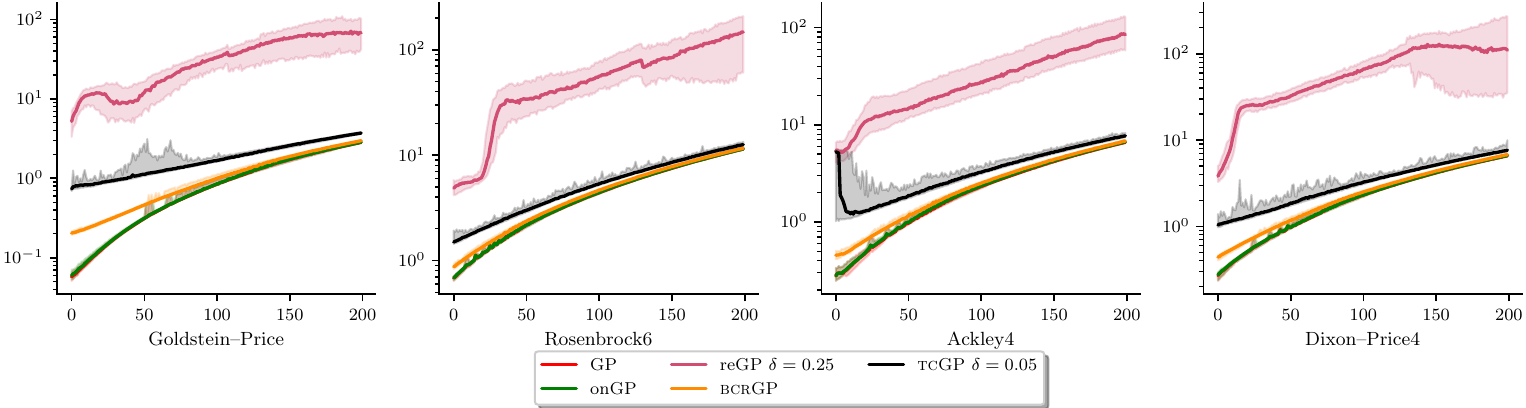}
        \caption{Computing time per BO iteration for GP, \tcgp, \onego, \bcrgp,
        and \regp, excluding maximization of the sampling criterion.}
\end{figure}

\subsection{Alternative Selection Criteria for \tcgp}
\label{ap:objective}

\begin{figure}[htbp]
    \centering
    \begin{subfigure}[b]{0.48\textwidth}
        \centering
        \includegraphics[width=\textwidth]{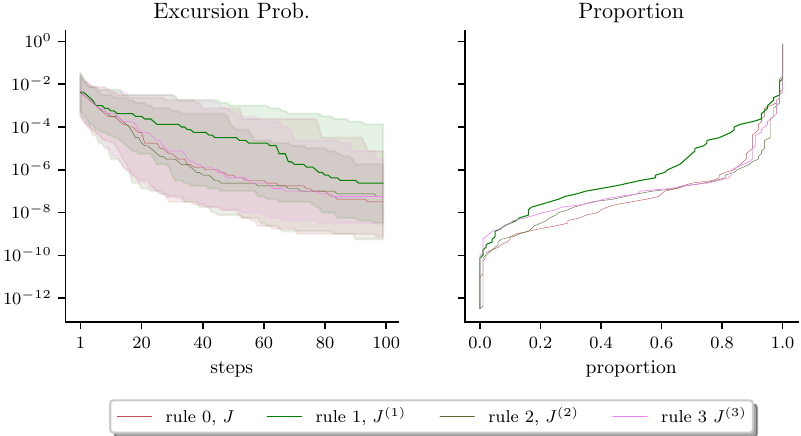}
        \caption{Ackley with $d=4$}
    \end{subfigure}
    \hfill
     \begin{subfigure}[b]{0.48\textwidth}
        \centering
        \includegraphics[width=\textwidth]{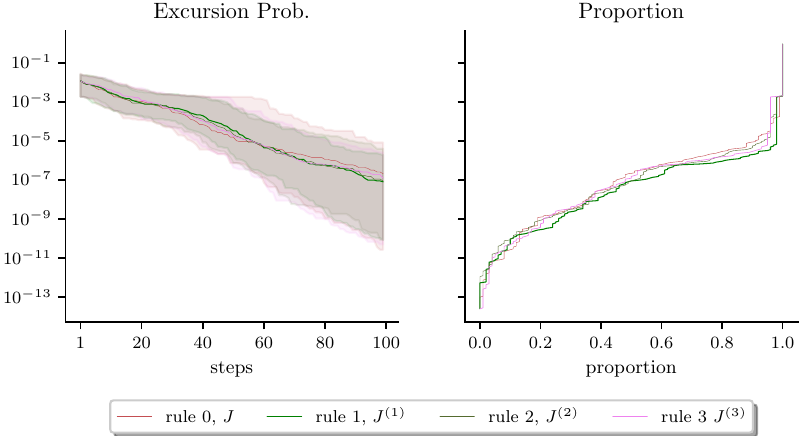}
        \caption{Dixon--Price with $d=4$}
    \end{subfigure}
    \hfill
    \begin{subfigure}[b]{0.48\textwidth}
        \centering
        \includegraphics[width=\textwidth]{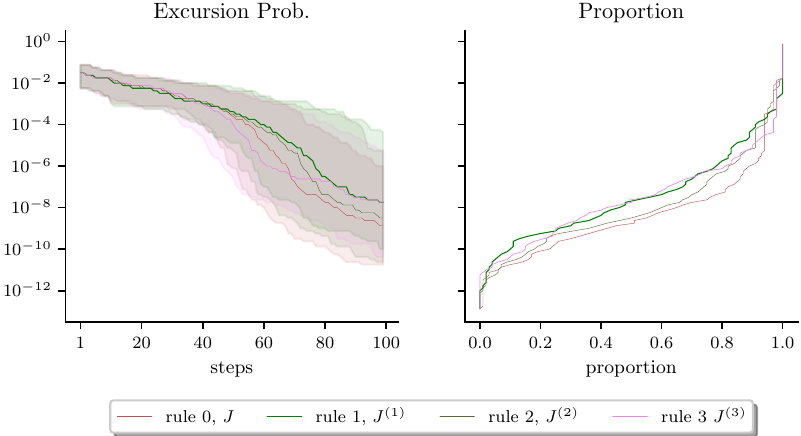}
        \caption{Goldstein--Price}
    \end{subfigure}
    \hfill
    \begin{subfigure}[b]{0.48\textwidth}
        \centering
        \includegraphics[width=\textwidth]{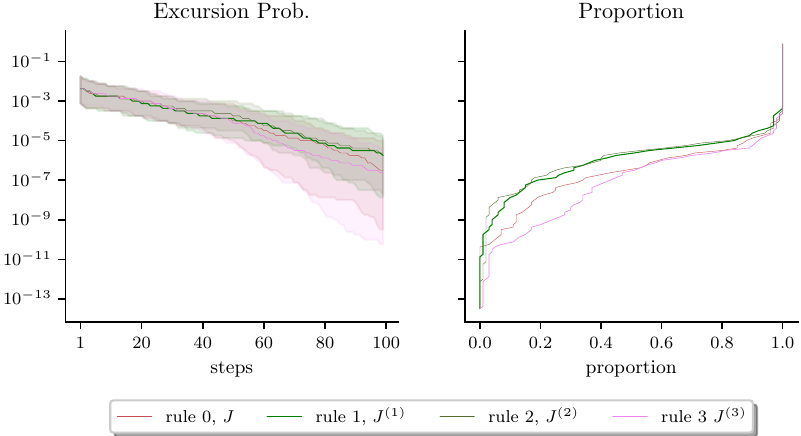}
        \caption{Rosenbrock with $d=6$}
    \end{subfigure}
    \caption{Comparison of \tcgp\ variants. Left: median and 10\%/90\% quantiles
    of $p_{m_n}=\P(f(X)\le m_n)$, where $m_n$ is the best observed value after $n$
    evaluations and $X\sim\mu$. Right: fraction of runs reaching the prescribed
    target value.}
    \label{fig:objective_function}
\end{figure}

We compare several objectives for selecting $(\beta,\lambda)$ in \tcgp, with the
aim of controlling thresholded $\mu$-calibration and occurrence calibration
below $t=q_{\delta,n}$. 
Rule~0 is the criterion $J$ defined in \eqref{eq:objective_tcgp}. 
Figure~\ref{fig:objective_function} reports BO results over $100$ iterations with
$\delta=0.1$.

A natural variant of \eqref{eq:objective_tcgp} replaces the target
$u\mapsto u\,\kappa_t^{\beta,\lambda}$ by the truncated map
\begin{equation}
u \mapsto \min \left(u\,\kappa_t^{\beta,\lambda},\,1\right).
\end{equation}
When $\kappa_t^{\beta,\lambda}>1$, truncation saturates the target at $1$.
Large values of $\kappa_t^{\beta,\lambda}$ then become weakly penalized, and
the objective can become insensitive to overestimation of lower-tail mass. We do
not use this variant.

\paragraph{Rule 1.}
Multiply the weighted LOO tKS--PIT metric by the occurrence discrepancy:
\begin{equation}
J^{(1)}(\beta,\lambda)
=
J^{\mathrm{LOO},w}_{\mathrm{tKS\text{-}PIT},n}\!\left(\hat F_n^{\beta,\lambda}\mid t\right)\, r^{\mathrm{LOO},w}_{t,n} .
\end{equation}
This product is zero whenever either component is zero and therefore does not
by itself enforce both calibration conditions.

\paragraph{Rule 2.}
Following \citet{allen:2025}, define
\begin{equation}
c_t(u)
=
\frac{\P_n\left(U_t^{\beta,\lambda}\le u \mid f(X)\le t\right)}
{\kappa_t^{\beta,\lambda}}.
\end{equation}
Then use the KS-type deviation
\begin{equation}
J^{(2)}(\beta,\lambda)
=
\sup_{u\in[0,1]}
\left|c_t(u)-u\right|.
\end{equation}

\paragraph{Rule 3.}
Compare the conditional tPIT curve to the rescaled map
$u\mapsto u/\kappa_t^{\beta,\lambda}$:
\begin{equation}
J^{(3)}(\beta,\lambda)
=
\sup_{u\in[0,1]}
\left|
\P_n\left(U_t^{\beta,\lambda}\le u \mid f(X)\le t\right)
-
\frac{u}{\kappa_t^{\beta,\lambda}}
\right|.
\end{equation}

\paragraph{Choice of objective.}
Figure~\ref{fig:objective_function} reports BO performance for the different
selection rules. Performance is summarized by~$p_{m_n}=\P(f(X)\le~m_n)$, where
$m_n$ is the best value observed after $n$ evaluations and $X\sim\mu$ (with $\mu$
uniform), and by the fraction of runs reaching the prescribed target.

Performance varies across test functions. On Ackley, Rule~1 gives a higher
median $p_{m_n}$ than Rules~0, 2, and 3, which are close. On Dixon--Price, the
four rules give similar results. On Goldstein--Price, Rule~0 gives the lowest
final median, whereas Rule~3 gives the lowest median on Rosenbrock.
We retain Rule~0 for the main experiments because its criterion directly
penalizes deviations of $\kappa_t^{\beta,\lambda}$ from $1$ without additional
tuning.

\subsection{Effect of the Choice of $\delta$ for \regp on BO Performance}
\label{ap:delta_on_regp}

\begin{figure}[htbp]
    \centering
    \begin{subfigure}[b]{0.48\textwidth}
        \centering
        \includegraphics[width=\textwidth]{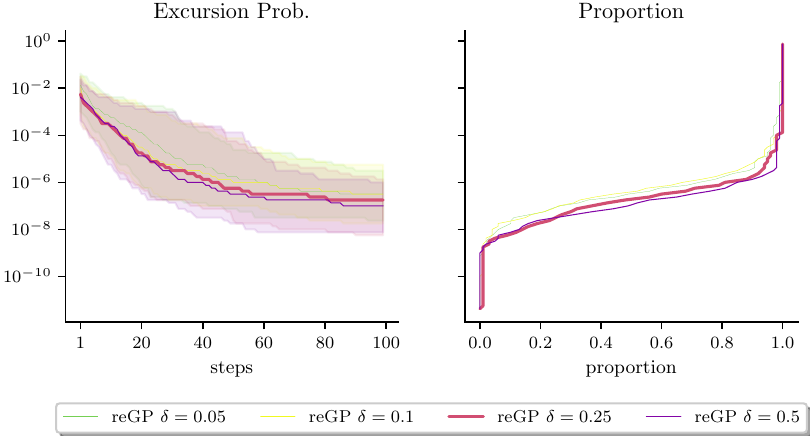}
        \caption{Ackley with $d=4$}
    \end{subfigure}
    \hfill
    \begin{subfigure}[b]{0.48\textwidth}
        \centering
        \includegraphics[width=\textwidth]{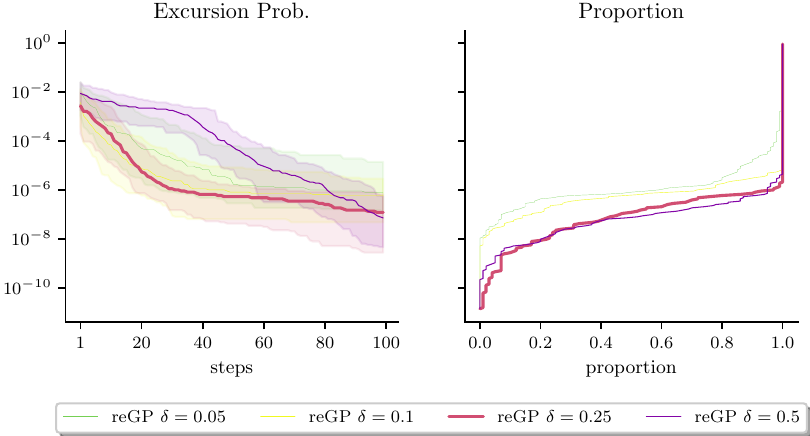}
        \caption{Dixon--Price with $d=4$}
    \end{subfigure}
    \hfill
    \begin{subfigure}[b]{0.48\textwidth}
        \centering
        \includegraphics[width=\textwidth]{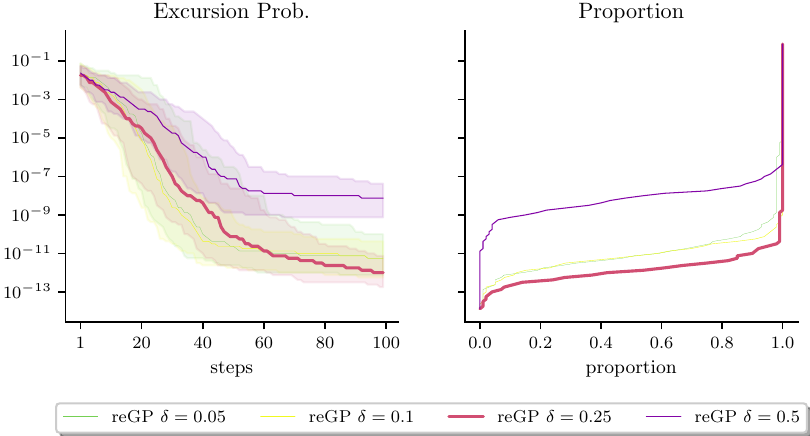}
        \caption{Goldstein--Price}
    \end{subfigure}
    \hfill
    \begin{subfigure}[b]{0.48\textwidth}
        \centering
        \includegraphics[width=\textwidth]{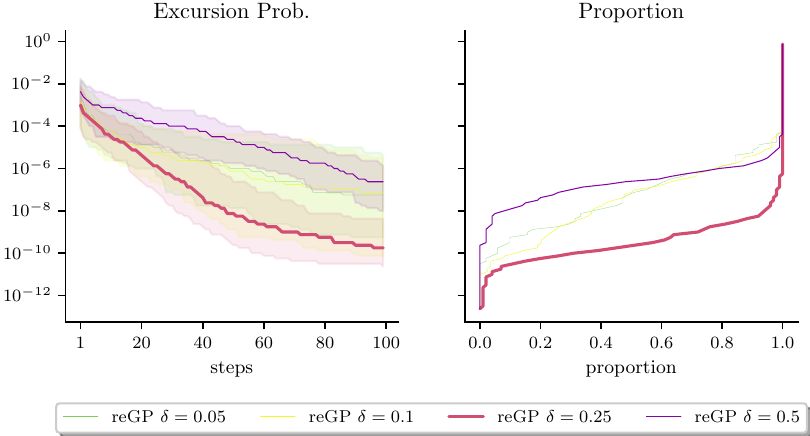}
        \caption{Rosenbrock with $d=6$}
    \end{subfigure}
    \caption{Sensitivity of \regp\ to $\delta$ (results shown after $n=100$ BO iterations).
    Left: median and 10\%/90\% quantiles of $p_{m_n}=~\P(f(X)\le m_n)$, where $m_n$ is the
    best value observed after $n$ evaluations and $X\sim\mu$. Right: fraction of runs
    reaching the prescribed target level.}
    \label{fig:delta_regp}
\end{figure}

Figure~\ref{fig:delta_regp} studies the sensitivity of \regp to the threshold
parameter $\delta$. BO performance is summarized by $p_{m_n}=\P(f(X)\le m_n)$,
where $m_n$ is the best value observed after $n$ evaluations and $X\sim\mu$ (with
$\mu$ uniform), and by the fraction of runs reaching the prescribed target.

Empirically, a large value of $\delta$ (e.g., $\delta=0.5$) consistently leads
to degraded performance across the considered problems. In contrast, $\delta=0.25$
achieves the best overall performance and is consistently comparable to, or better
than, smaller values such as $\delta=0.05$ and $\delta=0.1$. We therefore use
$\delta=0.25$ for all \regp experiments.

\subsection{Effect of the Choice of $\delta$ for \tcgp\ on BO Performance}
\label{ap:delta_on_tcgp}

\begin{figure}[htbp]
    \centering
    \begin{subfigure}[b]{0.48\textwidth}
        \centering
        \includegraphics[width=\textwidth]{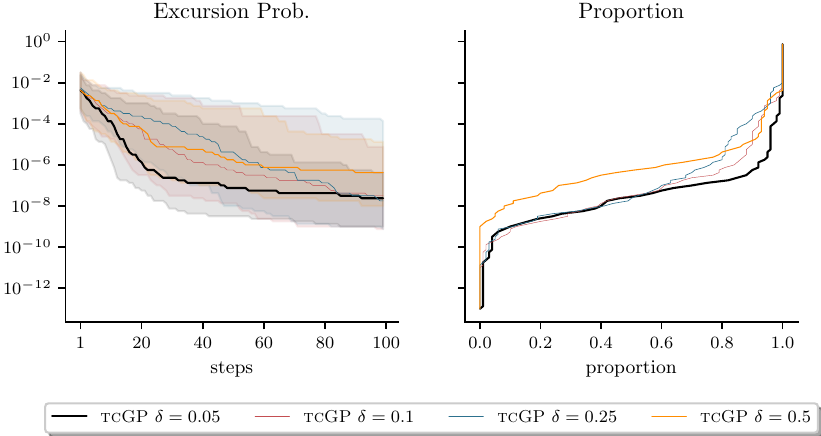}
        \caption{Ackley with $d=4$}
    \end{subfigure}
    \hfill
    \begin{subfigure}[b]{0.48\textwidth}
        \centering
        \includegraphics[width=\textwidth]{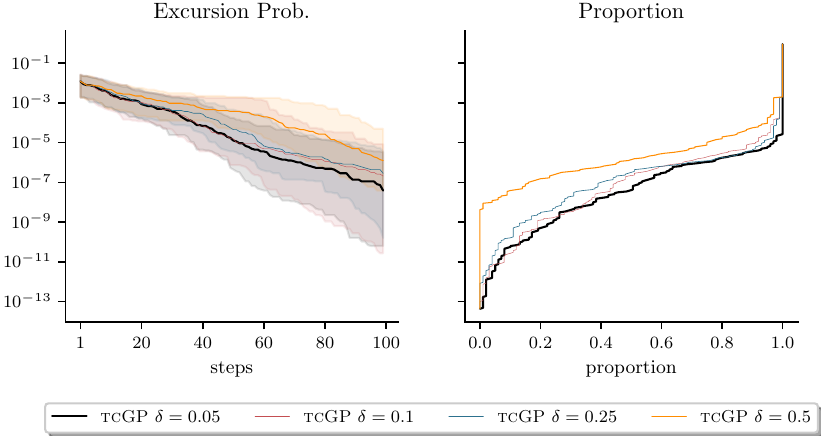}
        \caption{Dixon--Price with $d=4$}
    \end{subfigure}
    \hfill
    \begin{subfigure}[b]{0.48\textwidth}
        \centering
        \includegraphics[width=\textwidth]{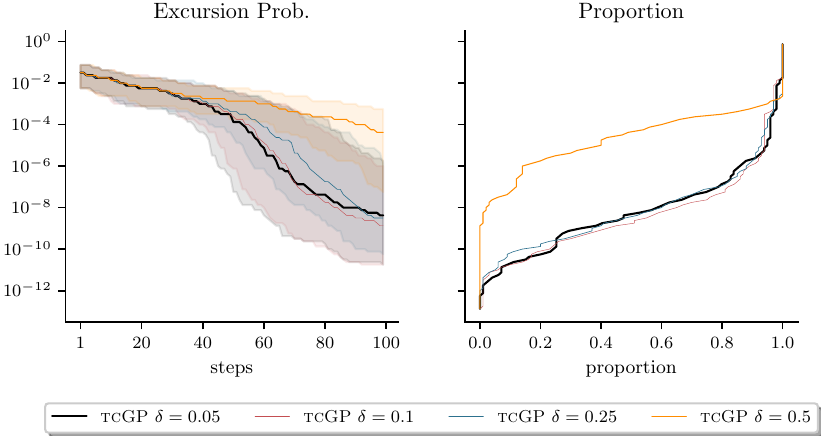}
        \caption{Goldstein--Price}
    \end{subfigure}
    \hfill
    \begin{subfigure}[b]{0.48\textwidth}
        \centering
        \includegraphics[width=\textwidth]{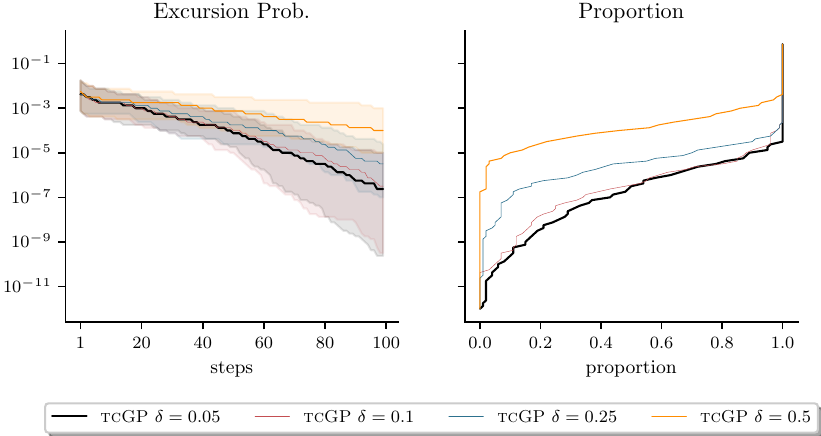}
        \caption{Rosenbrock with $d=6$}
    \end{subfigure}
    \caption{Sensitivity of \tcgp\ to $\delta$ (results shown after $n=100$ BO iterations).
    Left: median and 10\%/90\% quantiles of $p_{m_n}=~\P(f(X)\le m_n)$, where $m_n$ is the
    best value observed after $n$ evaluations and $X\sim\mu$. Right: fraction of runs
    reaching the prescribed target level.}
    \label{fig:delta_tcgp}
\end{figure}

Figure~\ref{fig:delta_tcgp} studies the influence of the threshold parameter
$\delta$ on \tcgp. Performance is summarized by $p_{m_n}=\P(f(X)\le m_n)$, where
$m_n$ is the best value observed after $n$ evaluations and $X\sim\mu$ (with $\mu$
uniform), and by the fraction of runs reaching the prescribed target.

Empirically, $\delta=0.5$ gives the highest final median $p_{m_n}$ on all four
functions. On Ackley, $\delta=0.05$ produces the fastest decrease, although
$\delta\in\{0.05,0.1,0.25\}$ gives similar medians at $n=100$. On
Dixon--Price, the final median decreases as $\delta$ decreases. On
Goldstein--Price, $\delta=0.1$ gives the lowest final median, although
the bands for $\delta\in\{0.05,0.1,0.25\}$ overlap. On Rosenbrock,
$\delta=0.05$ and $0.1$ give lower final medians than $\delta=0.25$ and $0.5$.
We therefore use $\delta=0.05$ for all \tcgp\ experiments.

\subsection{Calibration on a Fixed Dataset}
\label{ap:fixed_n_exp}

\subsubsection{Calibration below the Prescribed Threshold}
\label{sec:calibr-below-prescr}

We evaluate whether \tcgp, either with the combined objective or with objectives
targeting only thresholded or occurrence calibration, improves
calibration at the operational threshold $t_n=q_{\delta,n}$ used to select
$(\beta,\lambda)$, with $\delta\in\{0.25,0.1,0.05\}$.

The GP setup matches Section~\ref{sec:experiment_calibration}, except that we
work with a fixed training dataset $\Dcal_n$ of size $n=30d$ obtained from
uniform sampling on $\X$. For each test function, we generate $100$ independent
datasets and report averages over the replicates.

For a given $\delta$, calibration and predictive quality below $q_{\delta,n}$
are assessed using three metrics: twCRPS (jointly measuring calibration and
sharpness below $q_{\delta,n}$), $r_{t,n}$ (occurrence calibration at $t_n$), and
tKS--PIT (thresholded $\mu$--calibration below~$t_n$). The tKS--PIT is computed
using $4000$ points sampled conditional on $f(x)\le q_{\delta,n}$ by subset simulation
\citep[see Section~2 of][]{bect2017bayesian}. In addition, $4000$ points are
sampled uniformly on $\X$ to estimate $r_{t,n}$ and twCRPS.

Results are reported in Table~\ref{tab:occ_tkspit_results}. Overall, \tcgp\
improves tKS--PIT for most functions and values of $\delta$. The exceptions
are Hartmann ($d=6$) and Ackley ($d=4$) at $\delta=0.25$, and Ackley at
$\delta=0.1$. The twCRPS changes little except on Goldstein--Price, where it
decreases for all three values of $\delta$. For Rosenbrock and Dixon--Price, the
GP baseline $r_{t,n}$ is of order $10^{-2}$ and $10^{-3}$, respectively, and \tcgp\
leaves these values nearly unchanged. The local variant \tcgploc\ improves or
matches tKS--PIT throughout, but can worsen $r_{t,n}$ (e.g., Ackley with $d=4$ at
$\delta\in\{0.1,0.05\}$). The marginal variant \tcgpmar\ often improves $r_{t,n}$,
but worsens it for Dixon--Price at all three values of $\delta$ and for Hartmann
at $\delta=0.05$. It degrades tKS--PIT on several test functions.

\begin{table}[htbp]
  \centering
  \setlength\tabcolsep{2pt}
  \footnotesize

\begin{tabularx}{\textwidth}{l>{\columncolor{gray!15}}X>{\columncolor{gray!15}}X>{\columncolor{gray!15}}X>{\columncolor{white}}X>{\columncolor{white}}X>{\columncolor{white}}X>{\columncolor{gray!15}}X>{\columncolor{gray!15}}X>{\columncolor{gray!15}}X>{\columncolor{white}}X>{\columncolor{white}}X>{\columncolor{white}}X>{\columncolor{gray!15}}X>{\columncolor{gray!15}}X>{\columncolor{gray!15}}X}
\toprule
\textbf{Method} & \multicolumn{3}{c}{Goldstein Price} & \multicolumn{3}{c}{Rosenbrock6} & \multicolumn{3}{c}{Hartmann6} & \multicolumn{3}{c}{Dixon Price 4} & \multicolumn{3}{c}{Ackley4} \\
~ & \tiny twCRPS & \tiny $r_{t,n}$ & \tiny tKS--PIT & \tiny twCRPS & \tiny $r_{t,n}$ & \tiny tKS--PIT & \tiny twCRPS & \tiny $r_{t,n}$ & \tiny tKS--PIT & \tiny twCRPS & \tiny $r_{t,n}$ & \tiny tKS--PIT & \tiny twCRPS & \tiny $r_{t,n}$ & \tiny tKS--PIT \\
\midrule \midrule
\multicolumn{16}{l}{\textit{$n = 30d$, $t=q_{\delta,n}$, $\delta = 0.25$}} \\
\midrule
\textsc{gp} & 6.9e2 & 0.04 & 0.64 & \fnsb{6.0e3} & 0.01 & 0.16 & \fnsb{0.049} & 0.095 & \fnsb{0.15} & 2.7e2 & \fnsb{0.0059} & 0.16 & \fnsb{0.12} & 0.093 & 0.14 \\
\tcgp & \fnsb{5.7e2} & \fnsb{0.036} & \fnsb{0.27} & \fnsb{6.0e3} & 0.01 & \fnsb{0.14} & 0.052 & 0.06 & 0.28 & \fnsb{2.6e2} & \fnsb{0.0059} & \fnsb{0.13} & \fnsb{0.12} & 0.027 & 0.24 \\
\tcgpmar & 6.1e2 & \fnsb{0.036} & 0.52 & 8.9e3 & \fnsb{0.0095} & 0.36 & 0.059 & \fnsb{0.046} & 0.79 & 4.3e2 & 0.0071 & 0.42 & 0.13 & \fnsb{0.018} & 0.51 \\
\tcgploc & 5.8e2 & \fnsb{0.036} & \fnsb{0.27} & \fnsb{6e3} & 0.01 & \fnsb{0.14} & 0.05 & 0.11 & \fnsb{0.15} & \fnsb{2.6e2} & \fnsb{0.0059} & \fnsb{0.13} & \fnsb{0.12} & 0.078 & \fnsb{0.13} \\
\midrule
\multicolumn{16}{l}{\textit{$n = 30d$, $t=q_{\delta,n}$, $\delta = 0.1$}} \\
\midrule
\textsc{gp} & 6.0e2 & 0.11 & 0.83 & \fnsb{2.8e3} & \fnsb{0.011} & 0.3 & \fnsb{0.027} & 0.018 & 0.27 & 1.1e2 & \fnsb{0.0047} & 0.22 & \fnsb{0.063} & 0.042 & \fnsb{0.19} \\
\tcgp & \fnsb{5.1e2} & \fnsb{0.065} & 0.48 & \fnsb{2.8e3} & \fnsb{0.011} & \fnsb{0.27} & \fnsb{0.027} & 0.021 & 0.24 & \fnsb{1.0e2} & 0.0051 & \fnsb{0.19} & 0.066 & 0.021 & 0.25 \\
\tcgpmar & 5.2e2 & \fnsb{0.065} & 0.68 & 3.4e3 & \fnsb{0.011} & 0.58 & 0.031 & \fnsb{0.01} & 0.67 & 1.6e2 & 0.0057 & 0.44 & 0.071 & \fnsb{0.012} & 0.52 \\
\tcgploc& \fnsb{5.1e2} & \fnsb{0.065} & \fnsb{0.41} & 2.9e3 & \fnsb{0.011} & \fnsb{0.27} & 0.031 & 0.096 & \fnsb{0.12} & \fnsb{1.0e2} & 0.0053 & \fnsb{0.19} & 0.085 & 0.12 & \fnsb{0.19} \\
\midrule
\multicolumn{16}{l}{\textit{$n = 30d$, $t=q_{\delta,n}$, $\delta = 0.05$}} \\
\midrule
\textsc{gp} & 5.8e2 & 0.16 & 0.91 & \fnsb{1.8e3} & 0.023 & 0.39 & \fnsb{0.016} & \fnsb{0.0088} & 0.39 & 56 & \fnsb{0.005} & 0.32 & \fnsb{0.042} & 0.014 & 0.29 \\
\tcgp \tiny $\delta=0.05$ & 5e2 & \fnsb{0.089} & 0.65 & \fnsb{1.8e3} & \fnsb{0.02} & 0.34 & \fnsb{0.016} & 0.016 & 0.18 & \fnsb{55} & 0.0054 & \fnsb{0.26} & 0.044 & 0.025 & 0.28 \\
\tcgpmar & \fnsb{4.9e2} & 0.09 & 0.75 & 2.1e3 & \fnsb{0.02} & 0.61 & 0.017 & 0.011 & 0.39 & 64 & 0.0056 & 0.48 & 0.045 & \fnsb{0.011} & 0.42 \\
\tcgploc & \fnsb{4.9e2} & \fnsb{0.089} & \fnsb{0.62} & \fnsb{1.8e3} & \fnsb{0.02} & \fnsb{0.31} & 0.023 & 0.09 & \fnsb{0.15} & 56 & 0.0059 & 0.27 & 0.14 & 0.16 & \fnsb{0.22} \\
\midrule
\bottomrule
\end{tabularx}
        \caption{Calibration below $t=q_{\delta,n}$, assessed using the twCRPS,
        $r_{t,n}$, and tKS--PIT, for $\delta\in\{0.25,0.1,0.05\}$. We compare
        \tcgp\ with the combined objective, \tcgploc\ with tKS--PIT,
        \tcgpmar\ with $r^{\mathrm{LOO},w}_{t,n}$, and the standard GP baseline.
        Within each $\delta$ block, the smallest value in each column is
        underlined.
        The twCRPS and $r_{t,n}$ are evaluated on a test grid
        $(X_i,f(X_i))_{i=1}^{L_0}$, with $L_0=4000$ and uniformly distributed
        $X_i$. The tKS--PIT is approximated on another test grid
        $(\tilde X_i,f(\tilde X_i))_{i=1}^{L_1}$, with $L_1=4000$, where
        $f(\tilde X_i)\le q_{\delta,n}$ and the $\tilde X_i$ are sampled using a
        subset simulation algorithm \citep{bect2017bayesian}.}
        \label{tab:occ_tkspit_results}
\end{table}

\subsubsection{Comparison between Multiple Methods}

We evaluate whether \regp\ and \tcgp\ improve calibration below the current best
value $m_n$ relative to a standard GP model. As above, we consider
a fixed training dataset $\Dcal_n$ of size $n=30d$ obtained from uniform
sampling on $\X$. Results are reported in Table~\ref{tab:comparison}. \regp\
attains the lowest value of each metric on Goldstein--Price,
Rosenbrock, and Dixon--Price. On Hartmann and Ackley, \tcgp\ attains the lowest
tKS--PIT, whereas every \regp\ variant has a higher tKS--PIT than the GP
baseline.

On Ackley, \tcgp\ gives the best tKS--PIT at $m_n$, with $r_{m_n,n}$ comparable to
\regp, while improvements at the calibration threshold $t=q_{\delta,n}$ (with
$\delta=0.05$) were more limited. This suggests that calibration below an
operational threshold $t_n$ can still translate into gains at more extreme
levels $t\le t_n$.

\begin{table}[htbp]
  \centering
  \setlength\tabcolsep{2pt}
  \footnotesize
\begin{tabularx}{\textwidth}{l>{\columncolor{gray!15}}X>{\columncolor{gray!15}}X>{\columncolor{gray!15}}X>{\columncolor{white}}X>{\columncolor{white}}X>{\columncolor{white}}X>{\columncolor{gray!15}}X>{\columncolor{gray!15}}X>{\columncolor{gray!15}}X>{\columncolor{white}}X>{\columncolor{white}}X>{\columncolor{white}}X>{\columncolor{gray!15}}X>{\columncolor{gray!15}}X>{\columncolor{gray!15}}X}
\toprule
\textbf{Method} & \multicolumn{3}{c}{Goldstein Price} & \multicolumn{3}{c}{Rosenbrock6} & \multicolumn{3}{c}{Hartmann6} & \multicolumn{3}{c}{Dixon Price 4} & \multicolumn{3}{c}{Ackley4} \\
~ & \tiny twCRPS & \tiny $r_{m_n,n}$ & \tiny tKS--PIT & \tiny twCRPS & \tiny $r_{m_n,n}$ & \tiny tKS--PIT & \tiny twCRPS & \tiny $r_{m_n,n}$ & \tiny tKS--PIT & \tiny twCRPS & \tiny $r_{m_n,n}$ & \tiny tKS--PIT & \tiny twCRPS & \tiny $r_{m_n,n}$ & \tiny tKS--PIT \\
\midrule \midrule
\multicolumn{16}{l}{\textit{$n = 30d$}} \\
\midrule
\textsc{gp} & 5.7e2 & 0.2 & 0.99 & 6.8e2 & 0.036 & 0.77 & 0.0016 & 0.0043 & 0.61 & 16 & 0.015 & 0.83 & \fnsb{0.012} & 0.0061 & 0.72 \\
\tcgp \tiny $\delta=0.05$ & 4.9e2 & 0.12 & 0.78 & 6.2e2 & 0.027 & 0.56 & \fnsb{0.0015} & 0.0033 & \fnsb{0.33} & 14 & 0.011 & 0.68 & \fnsb{0.012} & 0.0047 & \fnsb{0.49} \\
\tcgp \tiny $\delta=0.1$ & 4.8e2 & 0.12 & 0.72 & 6.3e2 & 0.03 & 0.62 & \fnsb{0.0015} & 0.004 & 0.59 & 15 & 0.012 & 0.74 & \fnsb{0.012} & 0.0059 & 0.66 \\
\tcgp \tiny $\delta=0.25$ & 4.5e2 & 0.13 & 0.87 & 6.2e2 & 0.031 & 0.67 & 0.0016 & 0.0044 & 0.83 & 14 & 0.012 & 0.73 & \fnsb{0.012} & 0.0063 & 0.64 \\
\regp  \tiny $\delta = 0.05$ & \fnsb{1.9} & 0.021 & 0.77 & \fnsb{39} & 0.0039 & 0.66 & 0.0042 & 0.0048 & 0.89 & \fnsb{5.3} & \fnsb{0.0047} & 0.74 & 0.013 & 0.0046 & 0.95 \\
\regp \tiny $\delta = 0.1$ & 2.9 & \fnsb{0.019} & \fnsb{0.61} & 54 & \fnsb{0.0035} & \fnsb{0.43} & 0.0038 & 0.0038 & 0.84 & 11 & 0.0055 & \fnsb{0.66} & 0.013 & \fnsb{0.0043} & 0.89 \\
\regp \tiny $\delta = 0.25$ & 6.7 & 0.029 & 0.69 & 4.1e2 & 0.018 & 0.79 & 0.0018 & \fnsb{0.0032} & 0.72 & 25 & 0.015 & 0.82 & \fnsb{0.012} & 0.006 & 0.76 \\
\midrule
\bottomrule
\end{tabularx}
        \caption{Comparison of calibration using the twCRPS, $r_{m_n,n}$, and tKS--PIT
        below $t = m_n$ for \tcgp with multiple $\delta$, \regp, and the
        standard GP. The calibration methods use thresholds based on
        $q_{\delta,n}$, while the metrics are evaluated at $m_n$.
        The smallest value in each column is underlined.
        The twCRPS and $r_{m_n,n}$ are evaluated on a test grid
        $(X_i, f(X_i))_{i=1}^{L_0}$, with
        $L_0=4000$, and the $X_i$s are uniformly distributed. tKS--PIT is
        approximated using another test grid
        $(\tilde X_i, f(\tilde X_i))_{i=1}^{L_1}$, with $L_1=4000$, where
        $f(\tilde X_i) \le m_n$ and the $\tilde X_i$ are sampled using a subset
        simulation algorithm \citep{bect2017bayesian}.}
        \label{tab:comparison}
\end{table}

\subsection{Additional Results for EI and UCB}
\label{app:add_results_samp_crit}

\subsubsection{Additional Results for EI}
\label{app:add_res_ei}
In this section, we report results on additional test functions in
Figures~\ref{fig:add_1} and~\ref{fig:add_2}. BO performance
is summarized using two complementary views: the attained-level probability
$p_{m_n}=\P(f(X)\le m_n)$, where $m_n$ is the best observed value after $n$
iterations and $X\sim\mu$ (here $\mu$ is uniform on $\X$), and the fraction
of runs that reach a prescribed target level.

Overall, \regp\ and \tcgp\ consistently improve upon the standard GP baseline
and typically outperform \bcrgp\ and \onego. An exception is Ackley ($d=10$),
where \onego\ achieves the best performance. On the smoother (nearly convex)
objectives, \regp\ and \tcgp\ are the top-performing methods; on the Perm
function, \tcgp\ yields only modest gains over the GP baseline compared to
\regp. On highly non-convex functions, \tcgp\ most often provides the largest
improvements.

Figure~\ref{fig:add_3} reports diagnostic experiments that extend the benchmarks
in the main text to moderate-dimensional test functions ($d=10$ to $d=20$) over
$150$ BO iterations.

\begin{figure}[htbp]
    \centering
    \begin{subfigure}[b]{0.49\textwidth}
        \centering
        \includegraphics[width=\textwidth]{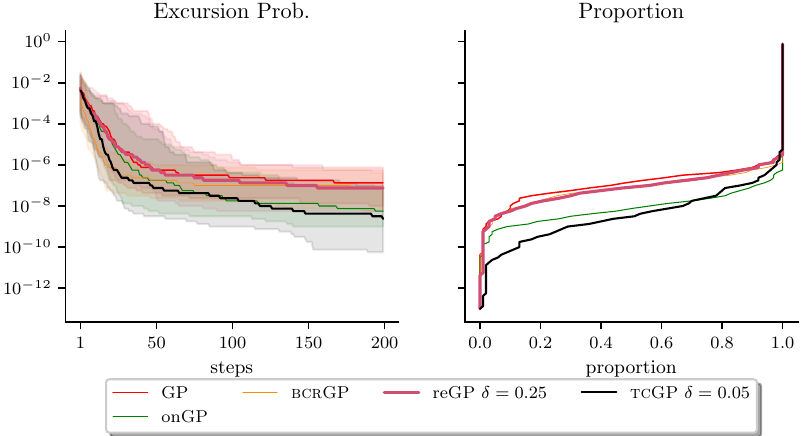}
        \caption{Ackley with $d=4$}
    \end{subfigure}
    \hfill
    \begin{subfigure}[b]{0.49\textwidth}
        \centering
        \includegraphics[width=\textwidth]{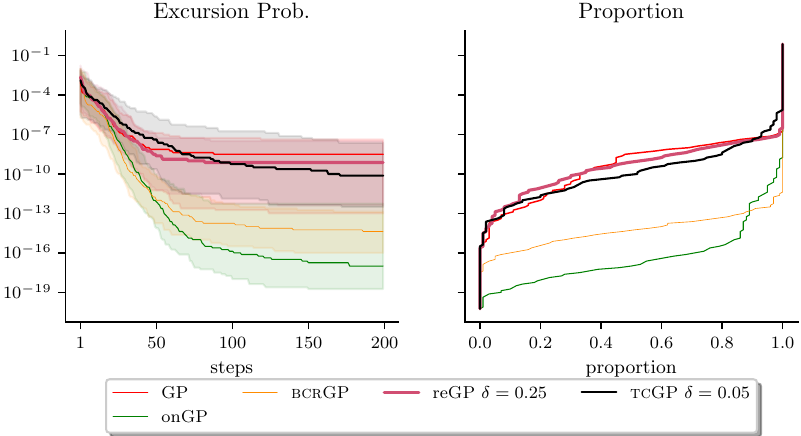}
        \caption{Ackley with $d=10$}
    \end{subfigure}
    \hfill
    \begin{subfigure}[b]{0.49\textwidth}
        \centering
        \includegraphics[width=\textwidth]{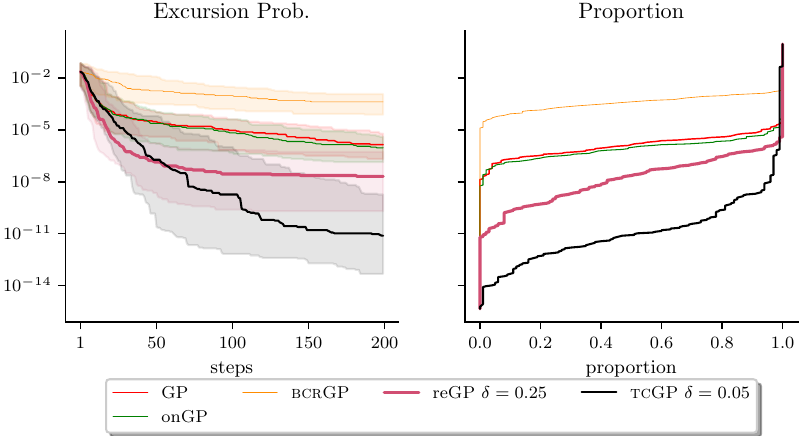}
        \caption{Cross-In-Tray}
    \end{subfigure}
    \begin{subfigure}[b]{0.49\textwidth}
        \centering
        \includegraphics[width=\textwidth]{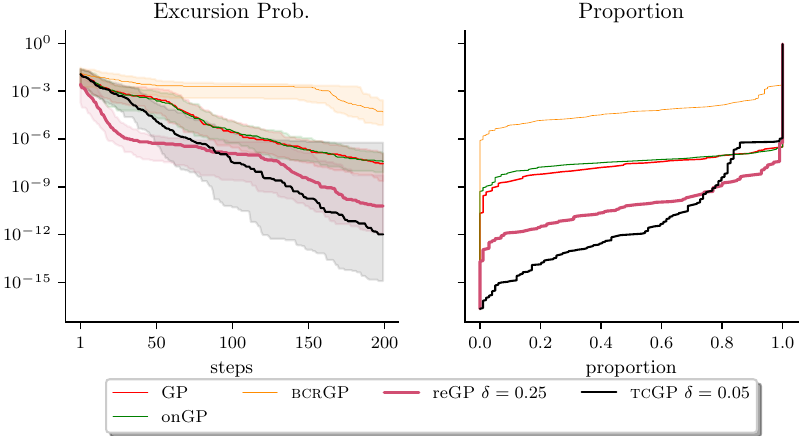}
        \caption{Dixon--Price with $d=4$}
    \end{subfigure}
    \begin{subfigure}[b]{0.49\textwidth}
        \centering
        \includegraphics[width=\textwidth]{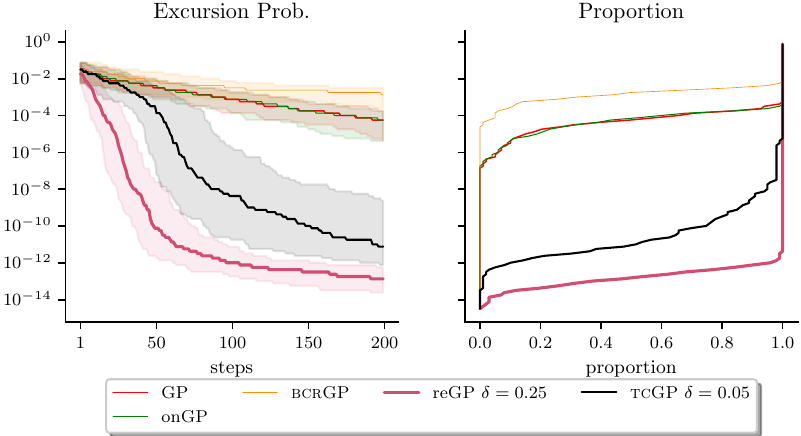}
        \caption{Goldstein--Price}
    \end{subfigure}
    \hfill
    \begin{subfigure}[b]{0.49\textwidth}
        \centering
        \includegraphics[width=\textwidth]{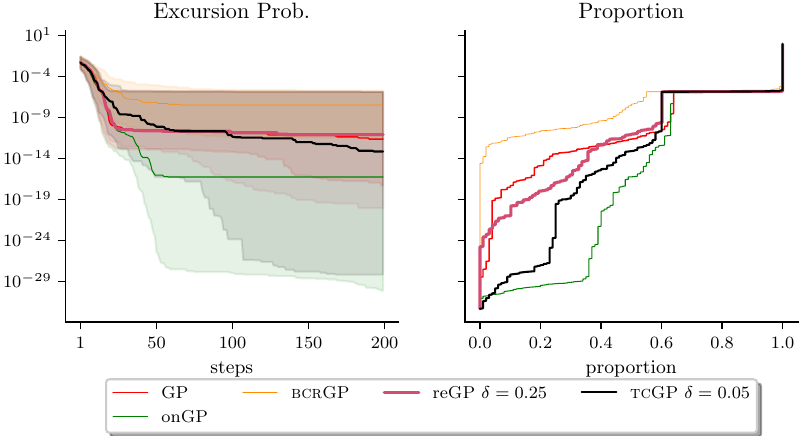}
        \caption{Hartmann with $d=6$}
    \end{subfigure}
    \caption{Comparison of BO performance using EI for GP, \onego, \bcrgp,
    \regp\ with $\delta=0.25$, and \tcgp\ with $\delta=0.05$. Left:
    median and 10\%/90\% quantiles of $p_{m_n} = \P(f(X) \le m_n)$, where
    $m_n$ is the best observed value so far. Right: fraction of
    successful runs reaching a prescribed target
    level.}
    \label{fig:add_1}
\end{figure}

\begin{figure}[htbp]
    \centering
    
    \hfill
    \begin{subfigure}[b]{0.49\textwidth}
        \centering
        \includegraphics[width=\textwidth]{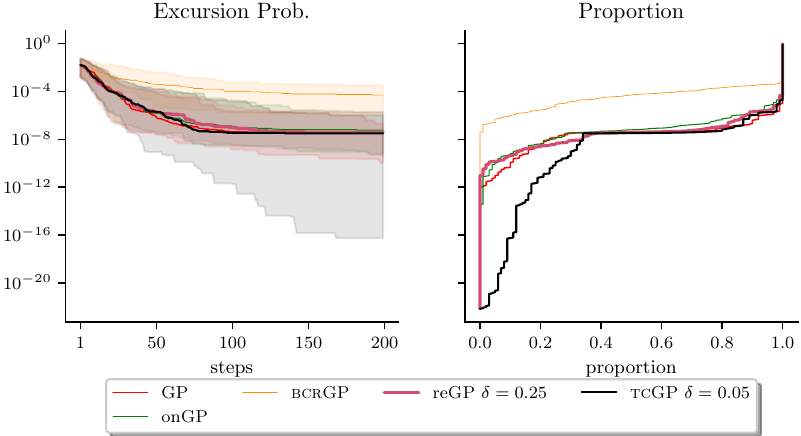}
        \caption{Michalewicz with $d=4$}
    \end{subfigure}
    \hfill
    \begin{subfigure}[b]{0.49\textwidth}
        \centering
        \includegraphics[width=\textwidth]{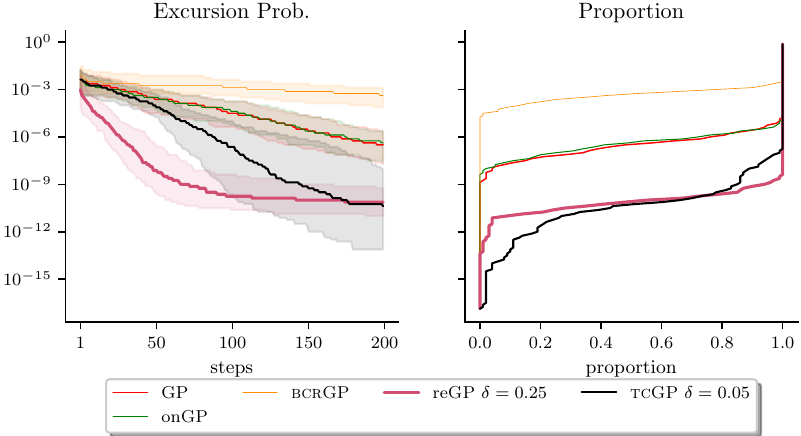}
        \caption{Rosenbrock with $d=6$}
    \end{subfigure}
    \hfill
    \begin{subfigure}[b]{0.49\textwidth}
        \centering
        \includegraphics[width=\textwidth]{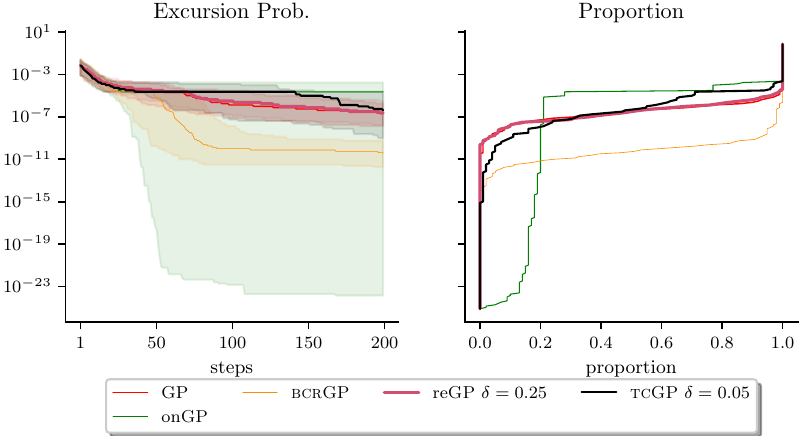}
        \caption{Shekel with $m=10$}
    \end{subfigure}
    \hfill
    \begin{subfigure}[b]{0.49\textwidth}
        \centering
        \includegraphics[width=\textwidth]{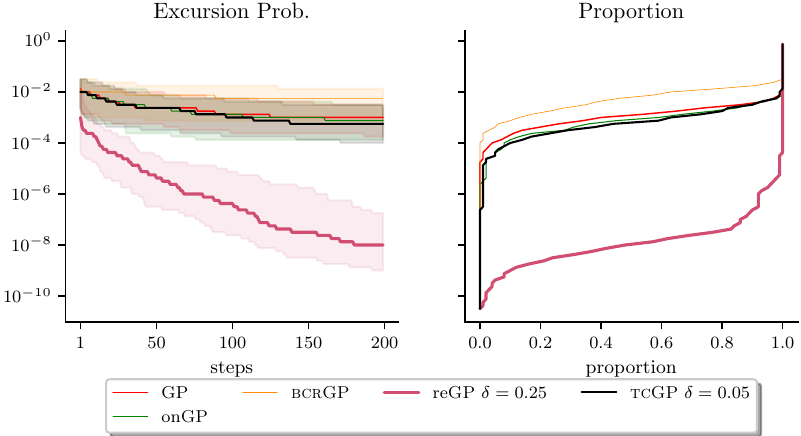}
        \caption{Perm with $d=6$}
    \end{subfigure}
    \caption{Comparison of BO performance using EI for GP, \onego, \bcrgp,
    \regp\ with $\delta=0.25$, and \tcgp\ with $\delta=0.05$. Left:
    median and 10\%/90\% quantiles of $p_{m_n} = \P(f(X) \le m_n)$, where
    $m_n$ is the best observed value so far. Right: fraction of
    successful runs reaching a prescribed target
    level.}
    \label{fig:add_2}
\end{figure}

\begin{figure}[htbp]
    \centering
    
    \hfill
    \begin{subfigure}[b]{0.49\textwidth}
        \centering
        \includegraphics[width=\textwidth]{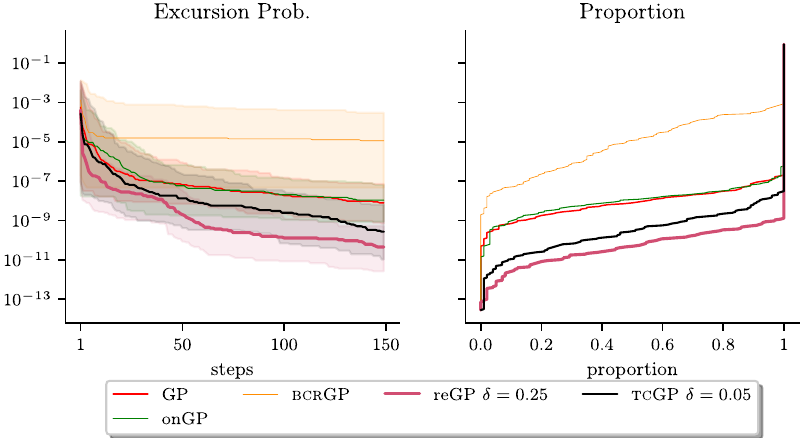}
        \caption{Dixon--Price with $d=10$}
    \end{subfigure}
    \hfill
    \begin{subfigure}[b]{0.49\textwidth}
        \centering
        \includegraphics[width=\textwidth]{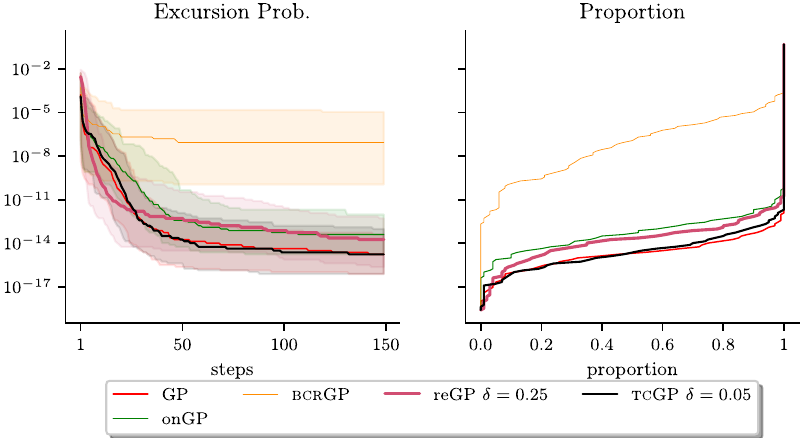}
        \caption{Dixon--Price with $d=20$}
    \end{subfigure}
    \hfill
    \begin{subfigure}[b]{0.49\textwidth}
        \centering
        \includegraphics[width=\textwidth]{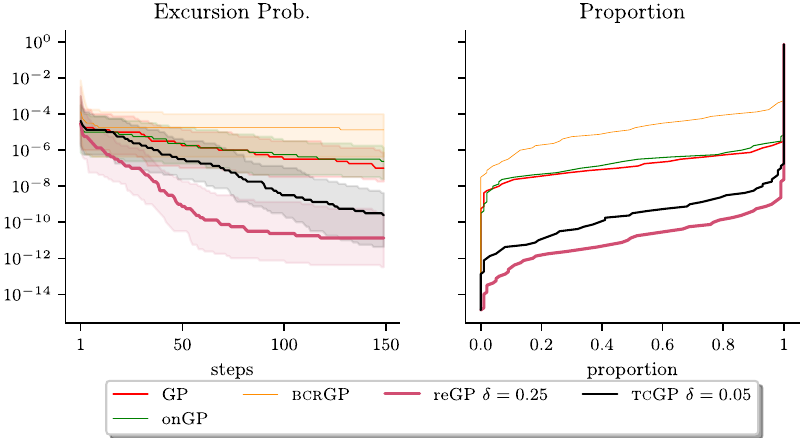}
        \caption{Rosenbrock with $d=10$}
    \end{subfigure}
    \hfill
    \begin{subfigure}[b]{0.49\textwidth}
        \centering
        \includegraphics[width=\textwidth]{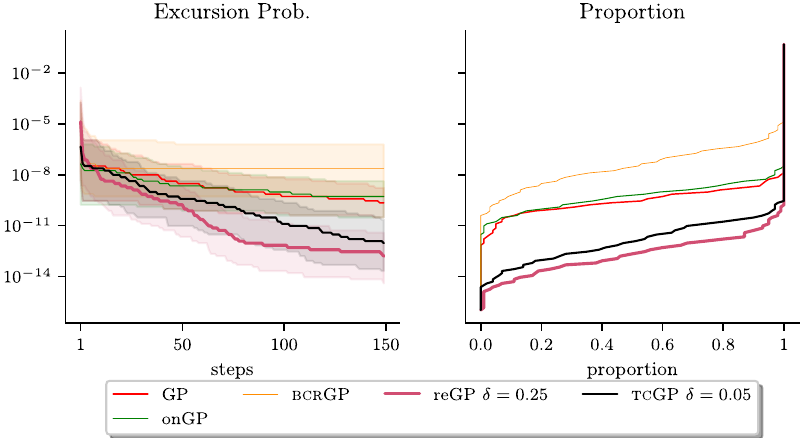}
        \caption{Rosenbrock with $d=15$}
    \end{subfigure}
    \caption{Comparison of BO performance in dimensions $d=10$ to $d=20$ over
    $150$ iterations using EI for GP, \onego, \bcrgp, \regp\ with
    $\delta=0.25$, and \tcgp\ with $\delta=0.05$. Left:
    median and 10\%/90\% quantiles of $p_{m_n} = \P(f(X) \le m_n)$, where
    $m_n$ is the best observed value so far. Right: fraction of
    successful runs reaching a prescribed target
    level.}
    \label{fig:add_3}
\end{figure}

\subsubsection{Results for UCB}
\label{app:add_res_ucb}

We compare UCB-based BO with a standard GP model, \onego, \bcrgp, \regp, and
\tcgp\ \citep{lairobbins1985, auer2002finite, srinivas2010:_ucb}, to assess
whether tail calibration also affects a criterion other than EI.
The GP and \tcgp\ forms are given below.

As above, BO performance is summarized by the probability of
excursion $p_{m_n}= \P(f(X) \le m_n)$, where $m_n$ is the best observed
value after $n$ iterations, and by the fraction of runs reaching a prescribed
target level, which provides a
complementary view of method reliability.

\paragraph{UCB with a GP model.}
Fix $\varepsilon\in(0,1/2)$. With a Gaussian predictive distribution, the
$(1-\varepsilon)$ lower confidence bound used for minimization is
\begin{equation}
    f_n(x) - \Phi^{-1}(1-\varepsilon)\,\sigma_n(x),
\end{equation}
where $\Phi$ denotes the standard normal CDF.

\paragraph{UCB with \tcgp.}
With \tcgp\ parameters $(\beta,\lambda)$, the predictive distribution is
generalized normal, and the corresponding $(1-\varepsilon)$ lower confidence
bound becomes
\begin{equation}
    f_n(x) - \Theta^{-1}_{\beta,0,1}(1-\varepsilon)\,\lambda\,\sigma_n(x),
\end{equation}
where $\Theta_{\beta,0,1}$ is the CDF defined in~\eqref{eq:bcrgp-cdf}.

Results are reported in Figures~\ref{fig:add_1_ucb} and~\ref{fig:add_2_ucb} 
with UCB and $\varepsilon=0.1$.
Overall, \tcgp\ outperforms the GP baseline on most test functions. On Ackley
($d=4$), the gains are modest, and the two methods give similar results for
$d=10$.

\begin{figure}[htbp]
    \centering
    \begin{subfigure}[b]{0.49\textwidth}
        \centering
        \includegraphics[width=\textwidth]{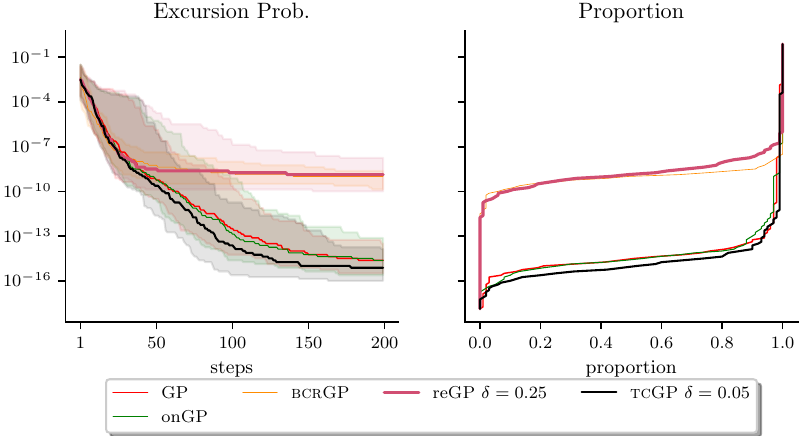}
        \caption{Ackley with $d=4$}
    \end{subfigure}
    \hfill
    \begin{subfigure}[b]{0.49\textwidth}
        \centering
        \includegraphics[width=\textwidth]{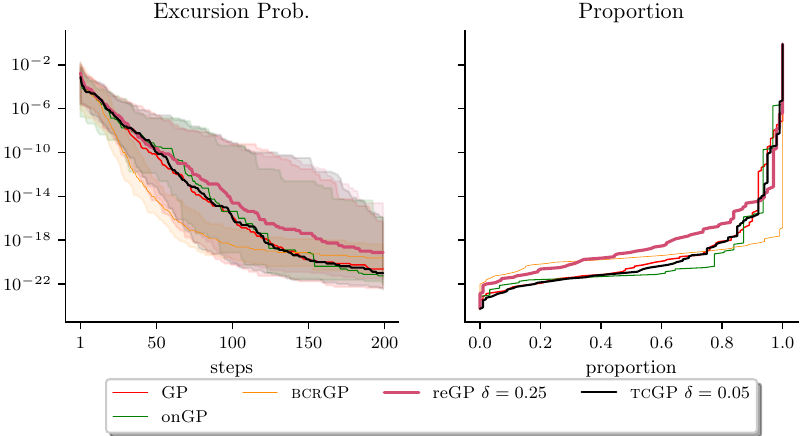}
        \caption{Ackley with $d=10$}
    \end{subfigure}
    \hfill
    \begin{subfigure}[b]{0.49\textwidth}
        \centering
        \includegraphics[width=\textwidth]{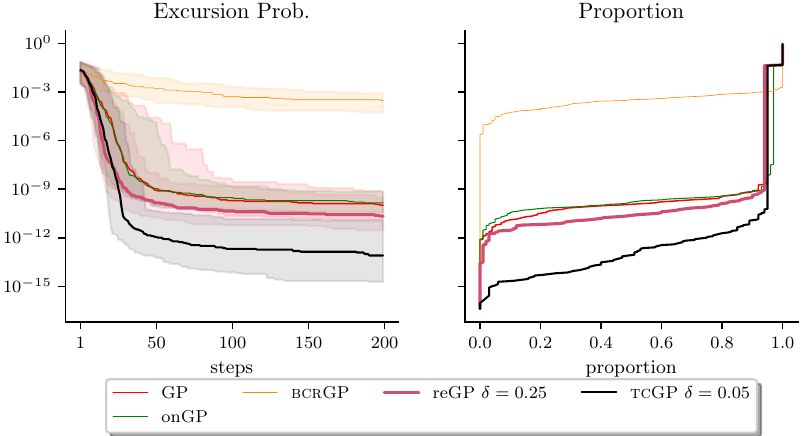}
        \caption{Cross-In-Tray}
    \end{subfigure}
    \begin{subfigure}[b]{0.49\textwidth}
        \centering
        \includegraphics[width=\textwidth]{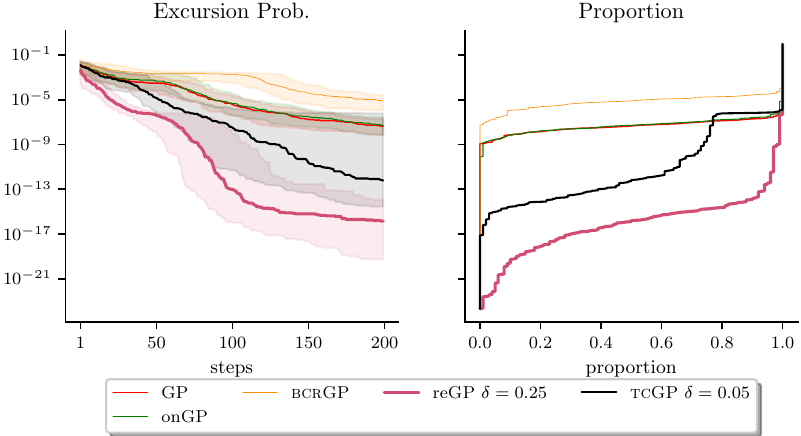}
        \caption{Dixon--Price with $d=4$}
    \end{subfigure}
    \begin{subfigure}[b]{0.49\textwidth}
        \centering
        \includegraphics[width=\textwidth]{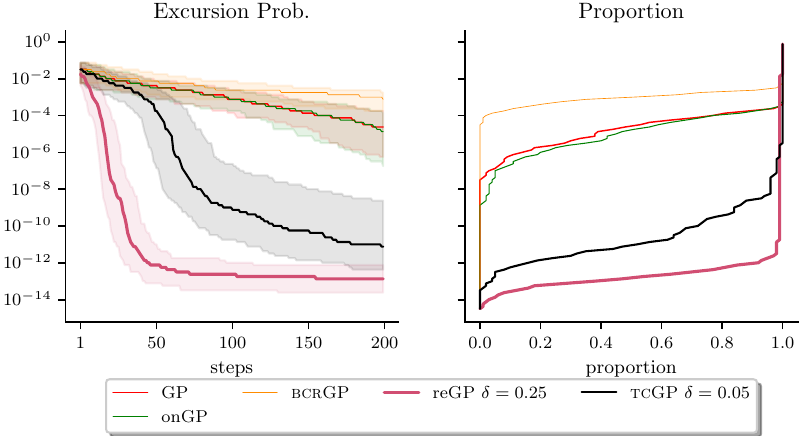}
        \caption{Goldstein--Price}
    \end{subfigure}
    \hfill
    \begin{subfigure}[b]{0.49\textwidth}
        \centering
        \includegraphics[width=\textwidth]{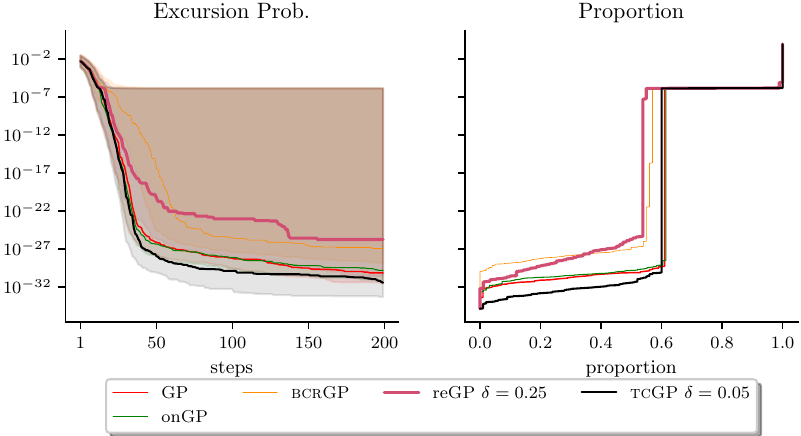}
        \caption{Hartmann with $d=6$}
    \end{subfigure}
    \caption{Comparison of BO performance using UCB with $\varepsilon=0.1$ for
    GP, \onego, \bcrgp, \regp\ with $\delta=0.25$, and \tcgp\
    with~$\delta=0.05$. Left:
    median and 10\%/90\% quantiles of $p_{m_n} = \P(f(X) \le m_n)$, where
    $m_n$ is the best observed value so far. Right: fraction of
    successful runs reaching a prescribed target
    level.}
    \label{fig:add_1_ucb}
\end{figure}

\begin{figure}[htbp]
    \centering
    
    \hfill
    \begin{subfigure}[b]{0.49\textwidth}
        \centering
        \includegraphics[width=\textwidth]{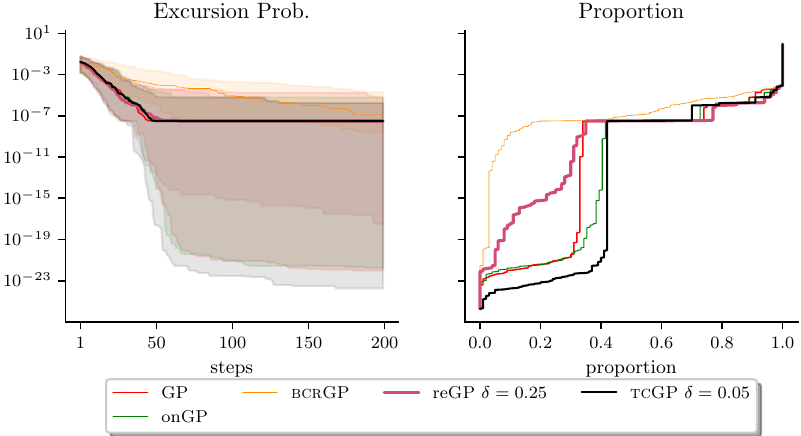}
        \caption{Michalewicz with $d=4$}
    \end{subfigure}
    \hfill
    \begin{subfigure}[b]{0.49\textwidth}
        \centering
        \includegraphics[width=\textwidth]{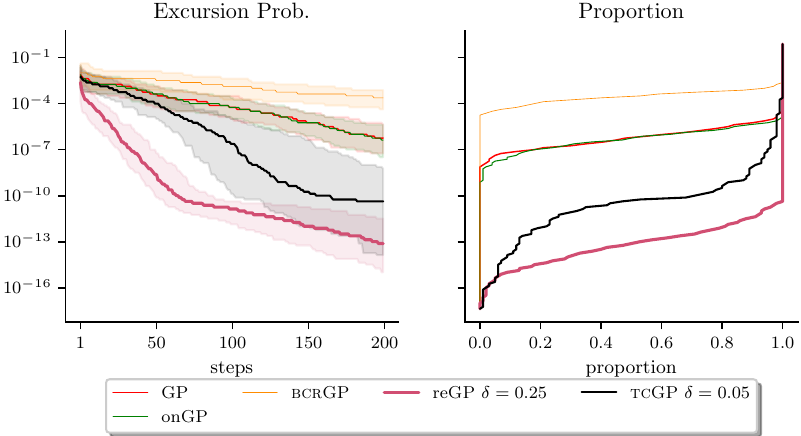}
        \caption{Rosenbrock with $d=6$}
    \end{subfigure}
    \hfill
    \begin{subfigure}[b]{0.49\textwidth}
        \centering
        \includegraphics[width=\textwidth]{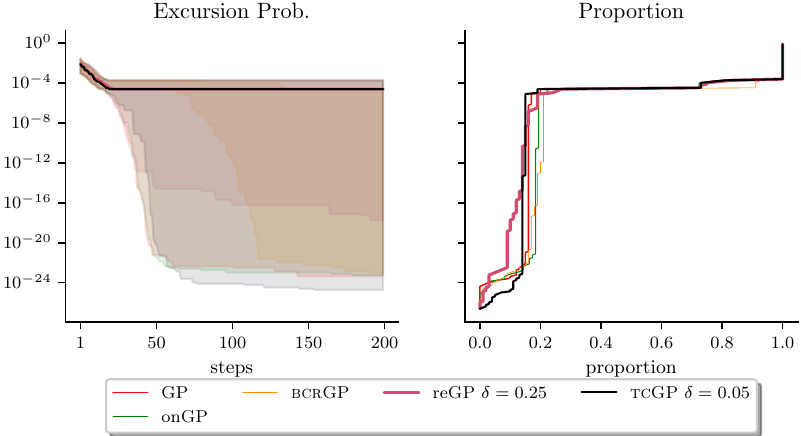}
        \caption{Shekel with $m=10$}
    \end{subfigure}
    \hfill
    \begin{subfigure}[b]{0.49\textwidth}
        \centering
        \includegraphics[width=\textwidth]{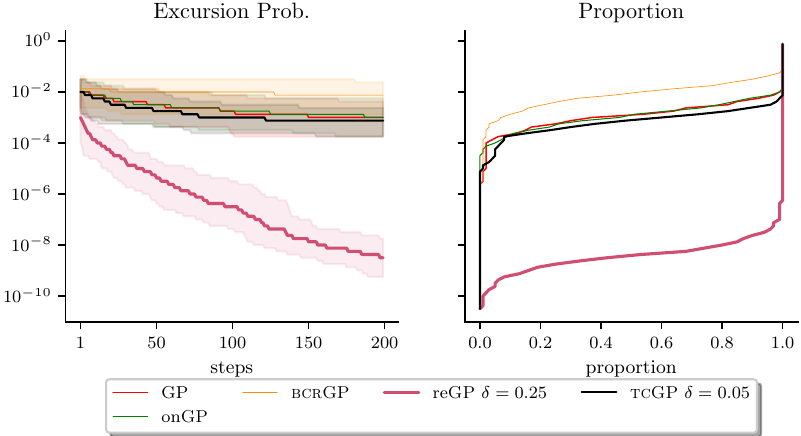}
        \caption{Perm with $d=6$}
    \end{subfigure}
    \caption{Comparison of BO performance using UCB with $\varepsilon=0.1$ for
    GP, \onego, \bcrgp, \regp\ with $\delta=0.25$, and \tcgp\
    with~$\delta=0.05$. Left:
    median and 10\%/90\% quantiles of $p_{m_n} = \P(f(X) \le m_n)$, where
    $m_n$ is the best observed value so far. Right: fraction of
    successful runs reaching a prescribed target
    level.}
    \label{fig:add_2_ucb}
\end{figure}

\end{document}